\documentclass[lettersize,journal]{IEEEtran}
\usepackage{amsmath,amsfonts}
\usepackage{algorithmic}
\usepackage{array}
\usepackage{caption}
\usepackage{textcomp}
\usepackage{stfloats}
\usepackage{url}
\usepackage{tabularx}
\usepackage{verbatim}
\usepackage{cite}
\usepackage{subfigure}
\usepackage{graphicx}
\usepackage{booktabs}
\usepackage{cases}
\usepackage{stmaryrd}
\usepackage{amssymb}
\usepackage[ruled,vlined]{algorithm2e}
\usepackage{threeparttable}
\usepackage{bbding}
\usepackage{mdframed} 

\newcommand{\BatchSMUL}{\texttt{BatchSMUL}}
\newcommand{\BatchSquare}{\texttt{BatchSquare}}
\newcommand{\SMIN}{\texttt{SMIN}}

\definecolor{cpcolor}{RGB}{245, 245, 245} 
\definecolor{cspcolor}{RGB}{215, 215, 215} 
\definecolor{whitecolor}{RGB}{253, 253, 253} 

\newtheorem{corollary}{\textbf{Corollary}}
\newtheorem{theorem}{\textbf{Theorem}}
\newtheorem{proof}{\textbf{Proof}}

\newcolumntype{Y}{>{\centering\arraybackslash}X}

\usepackage{hyperref}
\hypersetup{colorlinks=true,citecolor = blue}

\makeatletter

\newcommand{\Rmnum}[1]{\expandafter\@slowromancap\romannumeral #1@}
\makeatother
\usepackage{multirow} 
\def\BibTeX{{\rm B\kern-.05em{\sc i\kern-.025em b}\kern-.08em
    T\kern-.1667em\lower.7ex\hbox{E}\kern-.125emX}}
\usepackage{balance}

\begin{document}
\title{Pura: An Efficient Privacy-Preserving Solution for Face Recognition}

\author{
Guotao Xu,
Bowen Zhao*,~\IEEEmembership{Member,~IEEE},
Yang Xiao*,~\IEEEmembership{Member,~IEEE},
Yantao Zhong,
Liang Zhai,
Qingqi Pei,~\IEEEmembership{Senior Member,~IEEE}
\thanks{}
\IEEEcompsocitemizethanks{
\IEEEcompsocthanksitem G. Xu and B. Zhao are with the Guangzhou Institute of Technology, and Shaanxi Key Laboratory of Blockchain and Secure Computing, Xidian University, Guangzhou, China. E-mail: guotaoxu@stu.xidian.edu.cn, bwinzhao@gmail.com
\IEEEcompsocthanksitem Yang Xiao is with the School of Cyber Engineering, and the Engineering Research Center of Trusted Digital Economy, Universities of Shaanxi Province, Xidian University, Xi’an, China. E-mail: yxiao@xidian.edu.cn
\IEEEcompsocthanksitem Yantao Zhong, China Resources Intelligent Computing Technology Co., Ltd., Guangzhou, China. E-mail: zhongyantao@126.com
\IEEEcompsocthanksitem Liang Zhai, Chinese Academy of Surveying and Mapping, Beijing, China. E-mail: zhailiang@casm.ac.cn
\IEEEcompsocthanksitem Qingqi Pei is with the State Key Laboratory of Integrated Service Networks, Xidian University, Xi’an, China. E-mail: qqpei@mail.xidian.edu.cn

Corresponding author: Bowen Zhao, Yang Xiao}
}

\markboth{Journal of \LaTeX\ Class Files,~Vol.~xx, No.~xx, xx~xx}%
{How to Use the IEEEtran \LaTeX \ Templates}

\maketitle
\IEEEdisplaynontitleabstractindextext

\begin{abstract}

Face recognition is an effective technology for identifying a target person by facial images. However, sensitive facial images raises privacy concerns. Although privacy-preserving face recognition is one of potential solutions, this solution neither fully addresses the privacy concerns nor is efficient enough. To this end, we propose an efficient privacy-preserving solution for face recognition, named Pura, which sufficiently protects facial privacy and supports face recognition over encrypted data efficiently. Specifically, we propose a privacy-preserving and non-interactive architecture for face recognition through the threshold Paillier cryptosystem. Additionally, we carefully design a suite of underlying secure computing protocols to enable efficient operations of face recognition over encrypted data directly. Furthermore, we introduce a parallel computing mechanism to enhance the performance of the proposed secure computing protocols. Privacy analysis demonstrates that Pura fully safeguards personal facial privacy. Experimental evaluations demonstrate that Pura achieves recognition speeds up to 16 times faster than the state-of-the-art.

\end{abstract}


\section{Introduction}
\label{sec:introduction}

\IEEEPARstart{F}{ace} recognition is a practical biometric technology with widespread applications across various fields, including attendance access control \cite{manjula2012face}, security \cite{lander2018use}, and finance \cite{pinthong2020face}. Technically, face recognition identifies and verifies individuals based on their unique facial features \cite{haji2016real}. These features are captured by cameras and transformed into feature vectors through a face recognition model \cite{turk1991eigenfaces}. The recognition result is then determined by measuring the similarity (e.g., Euclidean distance, cosine similarity) between a feature vector and stored database entries \cite{perlibakas2004distance}.

Outsourcing face recognition tasks is a common way to enhance flexibility \cite{muslim2017face}, however, it raises significant privacy concerns. To outsource a face recognition task, the client device (e.g., smartphone, camera) captures personal facial images and sends the images or feature vectors to a cloud server, where the cloud server executes face recognition computations and returns the recognition result \cite{pawle2013face, siregar2018human}. Outsourcing face recognition tasks reduces the computational burden on the client device by utilizing the high-performance computing capabilities of the cloud server. Nevertheless, this practice introduces considerable privacy risks. The transmission and storage of sensitive data (i.e., facial images, feature vectors) by the cloud server may lead to data leakage \cite{laishram2024toward}. A notable example is the iCloud celebrity photo leak incident, which results in the unauthorized release of personal images \cite{fallon2015celebgate}. In various privacy-sensitive scenarios, leaky images or feature vectors threaten personal information security \cite{staicu2019leaky}.

Privacy-preserving face recognition (PPFR) is a potential solution that balances the strengths and concerns of outsourcing recognition tasks. A PPFR scheme enables the cloud server to perform face recognition without leaking facial data including facial images and feature vectors. To safeguard privacy, PPFR usually protects facial data through homomorphic encryption \cite{ogburn2013homomorphic}, secret sharing \cite{beimel2011secret}, or differential privacy \cite{xiong2020comprehensive}. In this case, the cloud server is enforced to perform face recognition operations on encrypted or confused data \cite{qin2014towards, guo2019towards}. However, current PPFR methods still face several limitations.

Existing PPFR solutions cannot always fully safeguard personal facial data and require interactions between the client and the cloud server. Current work, such as \cite{erkin2009privacy, sadeghi2009efficient, lei2021privface} requires client-server interactions during the recognition computation process, introducing computational and communication burdens on the client side. Furthermore, current PPFR solutions leak facial data like feature vectors to the cloud server. Most PPFR solutions adopt homomorphic encryption (HE) to support direct computations over encrypted data. Unfortunately, some HE-based PPFR solutions \cite{evans2011efficient, xiang2016privacy, drozdowski2019application} can not fully safeguard facial data because the cloud server obtains the plaintext results of the intermediate calculations, which means the cloud server can infer more information, such as feature vectors.

Moreover, existing PPFR approaches struggle to support effective recognition computations while maintaining recognition accuracy. On the one hand, PPFR solutions \cite{lei2021privface, drozdowski2019application} introduce new approaches for computing the similarity, such as the square of Euclidean distance and cosine similarity. However, they fall short in effectively performing the comparison operation, as they rely on plaintext for comparison. On the other hand, maintaining the same accuracy as schemes without privacy protection remains a difficulty. For instance, the solution \cite{blanton2012secure} employs approximation during computation, which results in a loss of accuracy. Additionally, some existing PPFR solutions \cite{chamikara2020privacy, ou2023faceidp} based on differential privacy enhance privacy protection by introducing noise, but this also compromises the usability of facial data and reduces recognition accuracy.

In general, existing PPFR solutions are limited by low efficiency due to complex computations and the inefficient use of computing resources. Although homomorphic encryption allows for computations over ciphertexts, recognition typically involves processing large amounts of data, which increases the computational overhead. To alleviate this dilemma, solutions \cite{evans2011efficient, drozdowski2019application} propose batch computations over ciphertexts. However, they still suffer from high computation costs, as operations like multiplication and minimum are still expensive and time-consuming. Furthermore, the distribution of server computing power is often suboptimal. In some multi-server PPFR solutions \cite{chun2014outsourceable, huang2023efficient}, one server always needs to wait for the calculation results from another one, which results in a waste of computational resources.

To avoid the above limitations, in this work, we innovatively propose an efficient privacy-preserving solution for face recognition, named Pura\footnote{Pura: an efficient \underline{p}rivacy-preserving sol\underline{u}tion fo\underline{r} f\underline{a}ce recognition}. Roughly speaking, we propose a privacy-preserving face recognition framework that is non-interactive and ensures no privacy leakage. To enable effective computations over ciphertexts without degrading accuracy, we design a suite of secure computing protocols. Additionally, by introducing a parallel computing mechanism, we improve the efficiency of the proposed Pura. The contributions of this paper are summarized in three folds:
\begin{itemize}
    \item We propose a novel framework for non-interactive and privacy-preserving face recognition falling in a twin-server architecture. The proposed framework enables cloud servers to perform privacy-preserving face recognition operations without frequent client-server interactions.

    \item We carefully design a suite of secure computing protocols that support effective recognition operations without sacrificing accuracy. The proposed protocols, including a batch secure square protocol (\BatchSquare) and secure minimum protocols (2-\SMIN~and $n$-\SMIN), enable valuable operations over ciphertexts, and support Pura to effectively perform the privacy-preserving face recognition task.

    \item We introduce an efficient parallel computing mechanism to enhance the computational efficiency of the servers. Specifically, both data storage and encrypted data computation are assigned to two servers simultaneously. Experimental evaluations demonstrate the efficiency of the proposed secure computing protocols and Pura.
\end{itemize}

The rest of this work is organized as follows. Section \ref{sec:related_work} briefly reviews the related work on privacy-preserving face recognition, and Section \ref{sec:preliminaries} then introduces preliminaries. After that, we define the system model and threat model in Section \ref{sec:models}. We carefully design a suite of secure computing protocols in Section \ref{sec:protocols}. In Section \ref{sec:solution}, we elaborate on the design of Pura, including problem formulation and the workflow of Pura. Rigorous privacy analysis is given in Section \ref{sec:analysis}. In Section \ref{sec:experiments}, the experimental results and a comprehensive analysis are shown. Finally, we summarize this work in Section \ref{sec:conclusion}.

\section{Related Work}
\label{sec:related_work}

In this section, we briefly review the related work on privacy-preserving face recognition.

Privacy-preserving face recognition has been extensively researched and has shown rapid development in recent years. Partially homomorphic encryption (PHE) allows specific computations over ciphertexts. Erkin \textit{et al.} \cite{erkin2009privacy} first integrated PHE into face recognition and proposed a privacy-preserving face recognition system. Subsequently, Sadeghi \textit{et al.} \cite{sadeghi2009efficient} combined PHE with garbled circuit (GC) to enhance the performance of the work \cite{erkin2009privacy}. Huang \textit{et al.} \cite{evans2011efficient} further optimized previous work of PPFR based on PHE and GC. However, the above solutions require client-server interactions during the recognition phase. Chun \textit{et al.} \cite{chun2014outsourceable} introduced a more secure and outsourceable privacy-preserving biometric authentication protocol. Nevertheless, this approach raises privacy concerns as the server gains knowledge of the recognition result, potentially revealing whether the input face matches any person in the database.

Since fully homomorphic encryption (FHE) supports more mathematical operations over encrypted data than PHE, several solutions based on FHE have been proposed lately. Xiang \textit{et al.} \cite{xiang2016privacy} firstly outsourced computation and presented a privacy-preserving face recognition protocol based on FHE. Although \cite{xiang2016privacy} encrypts the input facial data during face recognition, it stores the face database in plaintext on the server. Drozdowski \textit{et al.} \cite{drozdowski2019application} presented a PPFR protocol for a three-party PPFR architecture using FHE. However, the private key is held by one of the servers, thus the server is able to decrypt the ciphertexts during computation. Huang \textit{et al.} \cite{huang2023efficient} proposed the state-of-the-art PPFR protocol based on FHE and GC. Unfortunately, the work \cite{huang2023efficient} utilizes the column-wise strategy when storing the face database, resulting in a time-consuming process for dynamic updates of the database.

To the best of our knowledge, there are privacy-preserving biometric recognition solutions also using secret sharing \cite{blanton2012secure} and differential privacy \cite{chamikara2020privacy} to protect privacy. However, in the work \cite{blanton2012secure}, the server obtains the input access pattern, which can retrieve a record from the database a record that matches the input face data. In the solution \cite{chamikara2020privacy}, recognition accuracy is degraded due to the noise introduced.

To summarize, existing PPFR solutions still suffer from privacy leakage and poor efficiency. Our work, Pura, addresses the aforementioned issues. Pura is non-interactive and efficient. Besides, Pura leaks nothing about facial data and maintains the same precision as schemes without privacy protection. Additionally, Pura supports the update of the face database dynamically and allows for batch computations in parallel.

\section{Preliminaries}
\label{sec:preliminaries}

In this section, we introduce the threshold Paillier cryptosystem \cite{lysyanskaya2001adaptive} and the traditional face recognition approach \cite{barnouti2016face}. To improve clarity, we denote $\llbracket m \rrbracket$ as the ciphertext of a plaintext $m$. In addition, "$\to$" means an output operation. And $\stackrel{\$}{\gets}$ represents a random selection operation.

\subsection{Threshold Paillier Cryptosystem}

In this work, we consider the (2,2)-threshold Paillier cryptosystem \cite{zhao2023soci}, which consists of the following polynomial time algorithms:

    \textbf{N Generation} (\texttt{NGen}): \texttt{NGen} takes a security parameter $\kappa$ as input and outputs ($N,P,Q,p,q$). Formally, \texttt{NGen} is comprised of the following steps:

    \begin{enumerate}
        \item \label{NGen1} Randomly generate $\frac{l(\kappa)}{2}$-bit odd primes $p$ and $q$, where $l(\kappa)$ represents the bit length of the private key of the Paillier cryptosystem and $l(\kappa) = 4\kappa$;
    
        \item Randomly generate ($\frac{n(\kappa)-l(\kappa)}{2}-1$)-bit odd integers $p'$ and $q'$, where $n(\kappa)$ refers to the bit length of the modulus $N$ for Paillier;
    
        \item Calculate $P=2pp'+1$ and $Q=2qq'+1$;
    
        \item If $p,q,p',q'$ are not co-prime, or $P$ or $Q$ is not a prime, then re-generate ($p, q, P, Q$) from step \ref{NGen1});
    
        \item Calculate $N=PQ$, then output ($N,P,Q,p,q$).
    \end{enumerate}

    \textbf{Key Generation} (\texttt{KeyGen}): \texttt{KeyGen} takes a security parameter $\kappa$ as input and generates a public key $pk$ and a private key $sk$. \texttt{KeyGen} calls \texttt{NGen} to get  ($N,P,Q,p,q$) first. Then, \texttt{KeyGen} sequentially calculates $\alpha=pq$, $\beta=\frac{(P-1)(Q-1)}{4pq}$, and $h=-y^{2\beta} \!\!\mod N$, where $y \stackrel{\$}{\gets} \mathbb{Z}_N^*$. The public key is denoted by $pk=(h,N)$ and the private key is denoted by $sk=\alpha$. Particularly, the private key $sk$ is split into two threshold keys, denoted by $sk_1$ and $sk_2$, such that $sk_1 + sk_2=0 \!\!\mod 2\alpha$ and $sk_1 + sk_2=1 \!\!\mod N$. 
    $sk_1$ is set as a random number with $\sigma$ bits and $sk_2=((2\alpha)^{-1}\!\!\mod N)\cdot(2\alpha)-sk_1+\eta\cdot2\alpha\cdot N$, where $\eta \in \mathbb{Z}$.

    \textbf{Encryption} (\texttt{Enc}): \texttt{Enc} inputs a plaintext $m \in \mathbb{Z}_N$ and $pk$, and outputs a ciphertext $\llbracket m \rrbracket \in \mathbb{Z}_{N^2}^*$, denoted as 
    \begin{equation}
        \begin{aligned}
            \llbracket m \rrbracket &\gets \texttt{Enc}(pk,m) \\
            &= (1+N)^m\cdot(h^r\!\!\mod N)^N \!\!\mod N^2, 
        \end{aligned}
    \end{equation}
    where $r \stackrel{\$}{\gets} \{0,1\}^{l(\kappa)}$.

    \textbf{Decryption} (\texttt{Dec}): \texttt{Dec} inputs a ciphertext $\llbracket m \rrbracket \in \mathbb{Z}_{N^2}^*$ and $sk$, and outputs a plaintext $m \in \mathbb{Z}_{N}$, formulated as 
    \begin{equation}
        \begin{aligned}
            m &\gets \texttt{Dec}(sk,\llbracket m \rrbracket) \\
            &= (\frac{(\llbracket m \rrbracket^{2\alpha}\!\!\!\mod N^2) - 1}{N} \!\!\!\mod N) \cdot(2\alpha)^{-1}\!\!\!\mod N.
        \end{aligned}
    \end{equation}

    \textbf{Partial Decryption} (\texttt{PDec}): \texttt{PDec} inputs a ciphertext $\llbracket m \rrbracket \in \mathbb{Z}_{N^2}^*$ and a partially private key $sk_i$ ($i \in \{1, 2\}$), and outputs a partially decrypted ciphertext $\llbracket M_i \rrbracket \in \mathbb{Z}_{N^2}^*$, defined as 
    \begin{equation}
        \llbracket M_i \rrbracket \gets \texttt{PDec}(sk_i,\llbracket m \rrbracket) = \llbracket m \rrbracket^{sk_i}\!\!\!\mod N^2.
    \end{equation}

    \textbf{Threshold Decryption} (\texttt{TDec}): \texttt{TDec} inputs a pair of partially decrypted ciphertexts $\llbracket M_1 \rrbracket, \llbracket M_2 \rrbracket \in \mathbb{Z}_{N^2}^*$, and outputs a plaintext $m \in \mathbb{Z}_N$, described as 
    \begin{equation}
        \begin{aligned}
            m &\gets\texttt{TDec}(\llbracket M_1 \rrbracket, \llbracket M_2 \rrbracket) \\
            &= \frac{(\llbracket M_1 \rrbracket \cdot \llbracket M_2 \rrbracket \!\!\mod N^2) - 1}{N}\!\!\!\mod N.
        \end{aligned}
    \end{equation}

The (2,2)-threshold Paillier cryptosystem supports additive homomorphism and scalar-multiplication homomorphism.
\begin{itemize}
    \item \textbf{Additive homomorphism}:\\
    $\texttt{Dec}(sk,\llbracket m_1 \rrbracket \cdot \llbracket m_2 \rrbracket) = \texttt{Dec}(sk,\llbracket m_1 + m_2 \!\!\mod N \rrbracket)$.

    \item \textbf{Scalar-multiplication homomorphism}:\\
    $\texttt{Dec}(sk,\llbracket m \rrbracket^u) = \texttt{Dec}(sk,\llbracket u \cdot m \!\!\mod N \rrbracket), u \in \mathbb{Z}_N$.
\end{itemize}

\subsection{Face Recognition}
\label{pre:fr}
Face recognition checks whether a probe face matches any person in a face database of multiple persons or not \cite{cherepanova2023deep}. In practice, the facial image is usually transformed into an $n$-dimensional feature vector extracted by a deep learning model, e.g., Facenet \cite{schroff2015facenet}. Face recognition usually includes the registration phase and the recognition phase \cite{kar2006multi, parmar2014face}.

The registration phase involves storing multiple feature vectors, also known as face templates, on the server.
Formally, an $n$-dimensional feature vector database of $\Upsilon$ persons is formulated as a matrix with $\Upsilon$ rows and $n+1$ columns:
$$
    \mathcal{D}_{\Upsilon\times (n+1)} =
    \begin{pmatrix}
        ID_1 & v_{1,1} & v_{1,2} & \cdots & v_{1,n} & \\
        ID_2 & v_{2,1} & v_{2,2} & \cdots & v_{2,n} & \\
        \vdots & \vdots & \vdots & \ddots & \vdots & \\
        ID_\Upsilon & v_{\Upsilon,1} & v_{\Upsilon,2} & \cdots & v_{\Upsilon,n} &
    \end{pmatrix},
$$
where $ID$ means the identification of a person.

The recognition phase consists of calculating the distances between the probe feature vector and the feature vector database, and comparing the minimum distance to a threshold value. A probe feature vector is denoted as $p = (p_1, p_2, \cdots, p_n)$. In this work, we utilize the square of Euclidean distance \cite{gawande2014face} as the distance metric. The square of Euclidean distances between the probe feature vector $p$ and all feature vectors in database $\mathcal{D}$ are defined as:
\begin{equation}
\label{eq:squareEuclidean}
d_i^2 = \sum_{j=1}^{n} (p_j-v_{i,j})^2,1\leq i \leq \Upsilon.
\end{equation}
After that, the validity of the recognition result is determined by searching for the minimum distance $d_{min}$ among $\{d_1^2, d_2^2, \cdots, d_{\Upsilon}^2\}$ and then comparing $d_{min}$ with a threshold. If $d_{min}$ is not greater than the threshold, the recognition is successful, otherwise, the recognition is regarded as failed.

\section{System Model and Threat Model}
\label{sec:models}

In this section, we illustrate our system model and threat model of Pura.

\subsection{System Model}

\begin{figure}[ht]
	\centering
	\includegraphics[scale=0.56]{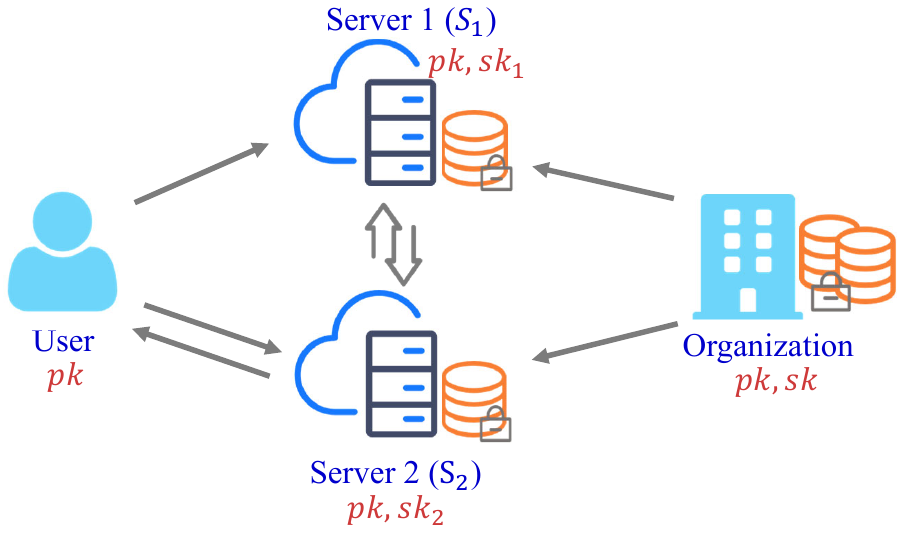}
	\caption{System model.}
	\label{fig:system_model}
\end{figure}

As depicted in Fig. \ref{fig:system_model}, our proposed Pura consists of an organization, multiple users, and two servers $S_1$ and $S_2$.
\begin{itemize}
    \item \textbf{Organization.} The organization owns a database of feature vectors. The organization calls \texttt{KeyGen} to generate a public key $pk$, a private key $sk$, and threshold keys ($sk_1, sk_2$). The organization encrypts the whole feature vector database under $pk$ and horizontally splits the encrypted database into two parts, denoted as $\llbracket \mathcal{D} \rrbracket_1$ and $\llbracket \mathcal{D} \rrbracket_2$. The organization distributes $pk$, ($pk, sk_1, \llbracket \mathcal{D} \rrbracket_1$), and ($pk, sk_2, \llbracket \mathcal{D} \rrbracket_2$) to the users, $S_1$, and $S_2$, respectively.

    \item \textbf{User}. The user is capable of capturing personal facial images. The user extracts a facial image into a feature vector $p$. To preserve privacy, $p$ is encrypted under $pk$ and is sent to both $S_1$ and $S_2$. Finally, the user obtains information from $S_2$ and recovers the recognition result.

    \item \textbf{Twin servers.} $S_1$ and $S_2$ both store a portion of the encrypted feature vector database. They are responsible for a part of the recognition computations, respectively. Meanwhile, they also provide online computation services for each other. When a user requests a recognition service, $S_1$ and $S_2$ jointly and parallelly perform face recognition operations over ciphertexts directly.
\end{itemize}

\subsection{Threat Model}

In this work, there is one type of adversary including the twin servers $S_1$ and $S_2$. Following previous work \cite{erkin2012generating, elmehdwi2014secure, mohassel2017secureml, zhao2022soci, zhao2023soci}, we suppose that $S_1$ and $S_2$ are semi-honest (or say honest-but-curious) and non-colluding. This means $S_1$ and $S_2$ follow the protocols strictly but they attempt to infer a user's private information, such as the face image and the recognition result. Besides, we consider the organization to be fully trusted and does not collude with other entities. We assume that the user is honest.

We consider the twin servers $S_1$ and $S_2$ to be provided by different or even competitive cloud service providers. Due to the commercial competition between service providers, $S_1$ and $S_2$ have no incentive to collude, as doing so would risk their own interests. The collusion between $S_1$ and $S_2$ implies that one server can access the private data of the other server, such as the partially private key, which will leave evidence for the competitor and even cause business losses. Therefore, collusion between the twin servers provided by different cloud service providers is unrealistic.


\section{Secure Computing Protocols}
\label{sec:protocols}

To support face recognition operations that do not disclose privacy, we carefully design a suite of secure computing protocols including batch secure multiplication (\BatchSMUL), batch secure square (\BatchSquare), secure 2-minimum (2-\SMIN), and secure $n$-minimum ($n$-\SMIN).

\subsection{\BatchSquare}
\label{subsec:mul}
FREED \cite{zhao2022freed} has constructed a \BatchSMUL~protocol based on the threshold Paillier cryptosystem. However, the work \cite{zhao2022freed} only supports batch secure multiplication over non-negative numbers. To extend batch secure multiplication to support negative numbers, we proposed a novel batch secure batch multiplication (\BatchSMUL) protocol.

We introduce a shared constant $\delta$ to convert the negative numbers into non-negative numbers. Note that, $\delta$ is set as 0 if all the plaintext values are non-negative, or else $\delta$ is set as the absolute value of the smallest number. \BatchSMUL~can perform secure multiplication for up to $\lfloor |N|/2|L| \rfloor$ pairs of ciphertexts at the same time \cite{zhao2022freed}, where $L$ is also a shared constant and $L \geq 2^{\sigma + 2}$, and $\sigma$ is a secure parameter and $2^\sigma \gg 2^\ell$, where $\ell$ is a constant that controls the domain size of plaintext. \BatchSMUL~takes two groups of ciphertexts $\llbracket x_1 \rrbracket, \cdots, \llbracket x_s \rrbracket$ and $\llbracket y_1 \rrbracket, \cdots, \llbracket y_s \rrbracket$ as inputs, where $s \leq \lfloor |N|/2|L| \rfloor$, $x_i, y_i \in (-2^\ell, 2^\ell)$, and $i \in [1,s]$, and outputs a result group by multiplying the input ciphertext groups in pairs, i.e., $\llbracket x_1 y_1 \rrbracket, \cdots, \llbracket x_s y_s\rrbracket$. As shown in Algorithm \ref{alg:BatchSMUL}, \BatchSMUL~consists of the following three steps:
\begin{enumerate}
    \item $S_1$ chooses two groups of random numbers $r_{i,1}, r_{i,2} \stackrel{\$}{\gets} \{0,1\}^\sigma$. Next, $S_1$ converts $x_i$ and $y_i$ into non-negative numbers and blinds them via the additive homomorphism as $\llbracket X_i \rrbracket \gets \llbracket x_i \rrbracket \cdot \llbracket \delta \rrbracket \cdot \llbracket r_{i,1} \rrbracket$ and $\llbracket Y_i \rrbracket \gets \llbracket y_i \rrbracket \cdot \llbracket \delta \rrbracket \cdot \llbracket r_{i,2} \rrbracket$. Then $S_1$ packs $\llbracket X_i \rrbracket$ and $\llbracket Y_i \rrbracket$ into one ciphertext by calculating 
    \begin{equation}
        \llbracket c_i \rrbracket = \llbracket X_i \rrbracket^{L^{2i-1}} \cdot \llbracket Y_i \rrbracket^{L^{2i-2}}, 
    \end{equation}
    and packs a group of ciphertexts $\llbracket c_i \rrbracket$ into one ciphertext by computing $\llbracket C \rrbracket \gets \prod_{i=1}^{i=s} \llbracket c_i \rrbracket$. Finally, $S_1$ partially decrypts $\llbracket C \rrbracket$ to get $\llbracket C_1 \rrbracket$ and sends $\langle \llbracket C \rrbracket, \llbracket C_1 \rrbracket \rangle$ to $S_2$.

    \item $S_2$ first partially decrypts $\llbracket C \rrbracket$ to get $\llbracket C_2 \rrbracket$ and then obtains 
    \begin{equation}
        \begin{aligned}
            C &= L^{2s-1}(x_s + r_{s, 1} + \delta) + L^{2s-2}(y_s + r_{s, 2} + \delta) \\
            &+ \dots + L(x_1 + r_{1, 1} + \delta) + (y_1 + r_{1, 2} + \delta)
        \end{aligned}
    \end{equation}
    via the threshold decryption. Next, $S_2$ unpacks $C$ to obtain $x_i + r_{i,1}$ and $y_i + r_{i,2}$ by computing
    \begin{equation}
        \begin{cases}
            c_i = \lfloor C \!\!\!\mod L^{2i} / L^{2(i-1)} \rfloor, \\
            x_i + r_{i,1} = \lfloor c_i / L \rfloor - \delta, \\
            y_i + r_{i,2} = c_i \!\!\!\mod L - \delta,
        \end{cases}
    \end{equation}
    Finally, $S_2$ encrypts $(x_i + r_{i,1}) \cdot (y_i + r_{i,2})$ and returns a group of ciphertexts $\{\llbracket (x_i + r_{i,1}) \cdot (y_i + r_{i,2}) \rrbracket\}_{i=1}^{i=s}$ to $S_1$.

    \item Since $S_1$ holds $r_{i,1}$, $r_{i,2}$, $\llbracket x_i \rrbracket$, and $\llbracket y_i \rrbracket$, it can locally compute $\llbracket -r_{i,2} (x_i + \delta) \rrbracket$, $\llbracket -r_{i,1} (y_i + \delta) \rrbracket$, $\llbracket -r_{i,1} \cdot r_{i,2} \rrbracket$, $\llbracket \delta (r_{i,1} + r_{i,2}) \rrbracket$, where $i \in [1, s]$. Then, according to the additive homomorphism, $S_1$ computes
    $\llbracket (x_i + r_{i,1}) \cdot (y_i + r_{i,2}) \rrbracket \cdot \llbracket -r_{i,2} (x_i + \delta) \rrbracket \cdot \llbracket -r_{i,1} (y_i + \delta) \rrbracket \cdot \llbracket -r_{i,1} \cdot r_{i,2} \rrbracket \cdot \llbracket \delta (r_{i,1} + r_{i,2}) \rrbracket$ to obtain $\llbracket x_i y_i \rrbracket$ for $i \in \{1, 2, \dots, s\}$.
\end{enumerate}

\begin{algorithm}[htbp]
	\caption{
        \begin{flushleft}
            \BatchSMUL($\langle\llbracket x_1 \rrbracket, \cdots, \llbracket x_s \rrbracket\rangle$, $\langle\llbracket y_1 \rrbracket, \cdots, \llbracket y_s \rrbracket\rangle$)
        \end{flushleft}
        \vspace{-5pt}
        \rightline{$\to \langle\llbracket x_1 y_1 \rrbracket, \cdots, \llbracket x_s y_s\rrbracket\rangle$ \hspace{-20pt}}
    }
        \label{alg:BatchSMUL}
        \vspace{3pt}
	\KwIn{$S_1$ holds $\langle\llbracket x_1 \rrbracket, \cdots, \llbracket x_s \rrbracket\rangle$ and $\langle\llbracket y_1 \rrbracket, \cdots, \llbracket y_s \rrbracket\rangle$.}
        \KwOut{$S_1$ obtains $\langle\llbracket x_1 y_1 \rrbracket, \cdots, \llbracket x_s y_s\rrbracket\rangle$.}


        \begin{mdframed}[backgroundcolor=cpcolor,innerleftmargin=2pt,innerrightmargin=2pt,innerbottommargin=2pt,leftmargin=-8pt,rightmargin=12pt]
            $\triangleright$ Step 1: $S_1$ computes
        \begin{mdframed}[backgroundcolor=whitecolor,innertopmargin=4pt,innerbottommargin=4pt,leftmargin=0pt,rightmargin=1pt,innerleftmargin=-4pt]
                \quad $\bullet$ $\llbracket X_i \rrbracket \gets \llbracket x_i \rrbracket \cdot \llbracket \delta \rrbracket \cdot \llbracket r_{i,1} \rrbracket$ and $\llbracket Y_i \rrbracket \gets \llbracket y_i \rrbracket \cdot \llbracket \delta \rrbracket \cdot \llbracket r_{i,2} \rrbracket$, 
                
                \quad\quad where $i \in [1, s]$;
    
                \quad $\bullet$ $\llbracket c_i \rrbracket = \llbracket X_i \rrbracket^{L^{2i-1}} \cdot \llbracket Y_i \rrbracket^{L^{2i-2}}$;
    
                \quad $\bullet$ $\llbracket C \rrbracket \gets \prod_{i=1}^{i=s} \llbracket c_i \rrbracket$ and $\llbracket C_1 \rrbracket \gets$ \texttt{PDec}($\llbracket C \rrbracket$);
    
                \quad $\bullet$ and sends $\langle \llbracket C \rrbracket, \llbracket C_1 \rrbracket \rangle$ to $S_2$.
        \end{mdframed}
        \end{mdframed}

        \vspace{-5pt}

        \begin{mdframed}[backgroundcolor=cspcolor,innerleftmargin=2pt,innerrightmargin=2pt,innerbottommargin=2pt,leftmargin=-8pt,rightmargin=12pt]
             $\triangleright$ Step 2: $S_2$ computes
        \begin{mdframed}[backgroundcolor=whitecolor,innertopmargin=4pt,innerbottommargin=4pt,leftmargin=0pt,rightmargin=1pt,innerleftmargin=-4pt]
                \quad $\bullet$ $\llbracket C_2 \rrbracket \gets$ \texttt{PDec}($\llbracket C \rrbracket$) and $C \gets$ \texttt{TDec}($\llbracket C_1 \rrbracket, \llbracket C_2 \rrbracket$);

                \quad $\bullet$ $c_i = \lfloor C \!\!\mod L^{2i} / L^{2(i-1)} \rfloor$, where $i \in [1,s]$;

                \quad $\bullet$ $x_i + r_{i,1} = \lfloor c_i / L \rfloor - \delta$ and $y_i + r_{i,2} = c_i \!\!\mod L - \delta$,
                
                \quad \quad for $i \in [1, s]$;
    
                \quad $\bullet$ $\llbracket (x_i + r_{i,1}) \cdot (y_i + r_{i,2}) \rrbracket \gets$ \texttt{Enc}($(x_i + r_{i,1}) \cdot (y_i + r_{i,2})$),
                
                \quad \quad for $i \in [1, s]$;
    
                \quad $\bullet$ and sends $\{\llbracket (x_i + r_{i,1}) \cdot (y_i + r_{i,2}) \rrbracket\}_{i=1}^{i=s}$ to $S_1$.
        \end{mdframed}
        \end{mdframed}

        \vspace{-5pt}

        \begin{mdframed}[backgroundcolor=cpcolor,innerleftmargin=2pt,innerrightmargin=2pt,innerbottommargin=2pt,leftmargin=-8pt,rightmargin=12pt]
            $\triangleright$ Step 3: For $i \in [1, s]$, $S_1$ computes
        \begin{mdframed}[backgroundcolor=whitecolor,innertopmargin=4pt,innerbottommargin=4pt,leftmargin=0pt,rightmargin=1pt,innerleftmargin=-4pt]
    
                \quad $\bullet$ $\llbracket -r_{i,2} (x_i + \delta) \rrbracket \gets \llbracket x_i + \delta \rrbracket^{-r_{i,2}}$, \\
                \hangindent=18pt $\llbracket -r_{i,1} (y_i + \delta) \rrbracket \gets \llbracket y_i + \delta \rrbracket^{-r_{i,1}}$, \\
                $\llbracket -r_{i,1} \cdot r_{i,2} \rrbracket \gets$ \texttt{Enc}($-r_{i,1} \cdot r_{i,2}$), \\
                $\llbracket \delta (r_{i,1} + r_{i,2}) \rrbracket \gets$ \texttt{Enc}($\delta (r_{i,1} + r_{i,2})$);
                
                \quad $\bullet$ $\llbracket x_i y_i \rrbracket \gets \llbracket (x_i + r_{i,1}) \cdot (y_i + r_{i,2}) \rrbracket \cdot \llbracket -r_{i,2} (x_i + \delta) \rrbracket \cdot$\\
                \hangindent=18pt $\llbracket -r_{i,1} (y_i + \delta) \rrbracket \cdot \llbracket -r_{i,1} \cdot r_{i,2} \rrbracket \cdot \llbracket \delta (r_{i,1} + r_{i,2}) \rrbracket$.
        \end{mdframed}
        \end{mdframed}


\end{algorithm}

Since it is more necessary to calculate the square of a ciphertext in our proposed design, we extend \BatchSMUL~to \BatchSquare. Compared to \BatchSMUL, \texttt{BatchSquare} is able to perform up to $\lfloor |N|/|L| \rfloor$ secure square calculations at the same time. \BatchSquare~takes a group of ciphertexts $\llbracket x_1 \rrbracket, \llbracket x_2 \rrbracket, \cdots, \llbracket x_s \rrbracket$ as inputs, where $i \in [1,s]$, $s \leq \lfloor |N|/|L| \rfloor$, and $-2^\ell < x_i < 2^\ell$, and outputs the ciphertexts of the square of each input ciphertext, i.e., $\llbracket x_1^2 \rrbracket, \cdots, \llbracket x_s^2\rrbracket$. \BatchSquare~is proposed in Algorithm \ref{alg:BatchSquare}. The main difference between \BatchSMUL~and \BatchSquare~lies in the objects and the method of packing. At the step 1, $S_1$ packs $\llbracket X_i \rrbracket$ by computing 
\begin{equation}
    \llbracket C \rrbracket \gets \prod_{i=1}^{i=s} \llbracket X_i \rrbracket^{L^{i-1}}. 
\end{equation}
At the step 2, $S_2$ unpacks $C$ through $\lfloor C\!\!\mod L^{i} / L^{i-1}  \rfloor$, and then obtains $x_i + r_i$ by subtracting $\delta$, where $i \in [1, s]$. At the step 3, $S_1$ obtains $\llbracket x_i^2 \rrbracket$ by computing $\llbracket x_i^2 \rrbracket \gets \llbracket (x_i + r_{i})^2 \rrbracket \cdot \llbracket -2r_{i} (x_i + \delta) \rrbracket \cdot \llbracket -r_{i}^2 \rrbracket \cdot \llbracket 2\delta r_{i} \rrbracket$ for $i \in \{1, 2, \dots, s\}$.

\begin{algorithm}[htbp]
	\caption{\BatchSquare$(\langle\llbracket x_1 \rrbracket, \cdots, \llbracket x_s \rrbracket\rangle)$\\
        \rightline{$\to \langle\llbracket x_1^2 \rrbracket, \cdots, \llbracket x_s^2 \rrbracket\rangle$ \hspace{-20pt}}}
        \label{alg:BatchSquare}
        \vspace{3pt}
	\KwIn{$S_1$ holds $\langle\llbracket x_1 \rrbracket, \cdots, \llbracket x_s \rrbracket\rangle$.}
        \KwOut{$S_1$ obtains $\langle\llbracket x_1^2 \rrbracket, \cdots, \llbracket x_s^2\rrbracket\rangle$.}

        \vspace{4pt}

        \begin{mdframed}[backgroundcolor=cpcolor,innerleftmargin=2pt,innerrightmargin=2pt,innerbottommargin=2pt,leftmargin=-8pt,rightmargin=12pt]
            $\triangleright$ Step 1: $S_1$ computes
        \begin{mdframed}[backgroundcolor=whitecolor,innertopmargin=4pt,innerbottommargin=4pt,leftmargin=0pt,rightmargin=1pt,innerleftmargin=-4pt]
                \quad $\bullet$ $\llbracket X_i \rrbracket \gets \llbracket x_i \rrbracket \cdot \llbracket \delta \rrbracket \cdot \llbracket r_{i} \rrbracket$, where $i \in [1, s]$;

                \vspace{2pt}
    
                \quad $\bullet$ $\llbracket C \rrbracket \gets \prod_{i=1}^{i=s} \llbracket X_i \rrbracket^{L^{i-1}}$ and $\llbracket C_1 \rrbracket \gets$ \texttt{PDec}($\llbracket C \rrbracket$);
    
                \quad $\bullet$ and sends $\langle \llbracket C \rrbracket, \llbracket C_1 \rrbracket \rangle$ to $S_2$.
        \end{mdframed}
        \end{mdframed}

        \vspace{-5pt}

        \begin{mdframed}[backgroundcolor=cspcolor,innerleftmargin=2pt,innerrightmargin=2pt,innerbottommargin=2pt,leftmargin=-8pt,rightmargin=12pt]
             $\triangleright$ Step 2: $S_2$ computes
        \begin{mdframed}[backgroundcolor=whitecolor,innertopmargin=4pt,innerbottommargin=4pt,leftmargin=0pt,rightmargin=1pt,innerleftmargin=-4pt]
                \quad $\bullet$ $\llbracket C_2 \rrbracket \gets$ \texttt{PDec}($\llbracket C \rrbracket$) and $C \gets$ \texttt{TDec}($\llbracket C_1 \rrbracket, \llbracket C_2 \rrbracket$);

                \quad $\bullet$ $x_i + r_{i} = \lfloor C\!\!\mod L^{i} / L^{i-1}  \rfloor - \delta$, where $i \in [1, s]$;
    
                \quad $\bullet$ $\llbracket (x_i + r_{i})^2 \rrbracket \gets$ \texttt{Enc}($(x_i + r_{i})^2$), where $i \in [1, s]$;
    
                \quad $\bullet$ and sends $\{\llbracket (x_i + r_{i})^2 \rrbracket\}_{i=1}^{i=s}$ to $S_1$.
        \end{mdframed}
        \end{mdframed}

        \vspace{-5pt}

        \begin{mdframed}[backgroundcolor=cpcolor,innerleftmargin=2pt,innerrightmargin=2pt,innerbottommargin=2pt,leftmargin=-8pt,rightmargin=12pt]
            $\triangleright$ Step 3: For $i \in [1, s]$, $S_1$ computes
        \begin{mdframed}[backgroundcolor=whitecolor,innertopmargin=4pt,innerbottommargin=4pt,leftmargin=0pt,rightmargin=1pt,innerleftmargin=-4pt]
    
                \quad $\bullet$ $\llbracket -2r_{i} (x_i + \delta) \rrbracket \gets \llbracket x_i + \delta \rrbracket^{-2r_{i}}$, \\
                \hangindent=18pt $\llbracket -r_{i}^2 \rrbracket \gets$ \texttt{Enc}($-r_{i}^2$), \\
                $\llbracket 2\delta r_{i} \rrbracket \gets$ \texttt{Enc}($2\delta r_{i}$);
                
                \quad $\bullet$ $\llbracket x_i^2 \rrbracket \gets \llbracket (x_i + r_{i})^2 \rrbracket \cdot \llbracket -2r_{i} (x_i + \delta) \rrbracket \cdot \llbracket -r_{i}^2 \rrbracket \cdot \llbracket 2\delta r_{i} \rrbracket$.
        \end{mdframed}
        \end{mdframed}


\end{algorithm}

\subsection{\SMIN}
\label{subsec:com}
Zhao \textit{et al.} \cite{zhao2023soci} have proposed a secure comparison protocol (\texttt{SCMP}). However, it is computationally expensive to extend \texttt{SCMP} to support the secure minimum protocol. Therefore, in this work, we propose a novel secure 2-minimum protocol and extend it to a secure $n$-minimum protocol. We first propose 2-\SMIN~that outputs the ciphertext of the minimum value of two given ciphertexts. Based on 2-\SMIN, we develop $n$-\SMIN~that outputs the ciphertext of the minimum value of $n$ given  ciphertexts.

Given two ciphertexts $\llbracket x \rrbracket$ and $\llbracket y \rrbracket$, where $x,y \in [-2^\ell, 2^\ell]$, 2-\SMIN~outputs the ciphertext of the minimum value between $x$ and $y$, i.e., 2-\SMIN($\llbracket x \rrbracket, \llbracket y \rrbracket$) $\to$ $\llbracket min(x,y) \rrbracket$. 2-\SMIN~is shown in Algorithm \ref{alg:2-SMIN}, which consists of three steps as follows:

\begin{enumerate}
    \item $S_1$ selects a random number $\pi \in \{0,1\}$ and computes
    \begin{equation}
        \llbracket D \rrbracket = 
        \begin{cases}
            (\llbracket x \rrbracket \cdot \llbracket y \rrbracket^{N - 1})^{r_1} \cdot \llbracket r_1 + r_2 \rrbracket, &\pi = 0 \\
            (\llbracket y \rrbracket \cdot \llbracket x \rrbracket^{N - 1})^{r_1} \cdot \llbracket r_2 \rrbracket, &\pi = 1
        \end{cases},
    \end{equation}
    where random numbers $r_1 \stackrel{\$}{\gets} \{0,1\}^\sigma $\textbackslash $\{0\}$ and $r_2$, s.t., $r_2 \leq \frac{N}{2}$ and $r_1 + r_2 > \frac{N}{2}$. $\sigma$ is a secure parameter, e.g., $\sigma=128$. Next, $S_1$ partially decrypts $\llbracket D \rrbracket$ to get $\llbracket D_1 \rrbracket$ and then sends $\langle \llbracket D \rrbracket, \llbracket D_1 \rrbracket, \llbracket x \rrbracket, \llbracket y \rrbracket \rangle$ to $S_2$.

    \item $S_2$ obtains $D$ through partially decryption and threshold decryption with $\llbracket D \rrbracket$ and $\llbracket D_1 \rrbracket$. Based on the value of $D$, $S_2$ assigns $\llbracket d_0 \rrbracket$ with a refreshed $\llbracket x \rrbracket$ or $\llbracket y \rrbracket$. That is, 
    \begin{equation}
        \llbracket d_0 \rrbracket = 
        \begin{cases}
            \llbracket y \rrbracket \cdot \llbracket 0 \rrbracket^r, &D>\frac{N}{2} \\
            \llbracket x \rrbracket \cdot \llbracket 0 \rrbracket^r, &D \leq \frac{N}{2}
        \end{cases},
    \end{equation}
    where $r$ is a random number used to refresh ciphertext, e.g., $r \stackrel{\$}{\gets} \{0,1\}^\ell$. $S_2$ returns $\llbracket d_0 \rrbracket$ to $S_1$.

    \item $S_1$ obtains the ciphertext of the minimum one between two ciphertexts according to the value of $\pi$ generated in Step 1. If $\pi=0$, $S_1$ sets $\llbracket min(x, y) \rrbracket = \llbracket d_0 \rrbracket$, otherwise ($\pi=1$), $S_1$ computes $\llbracket min(x, y) \rrbracket = \llbracket x \rrbracket \cdot \llbracket y \rrbracket \cdot \llbracket d_0 \rrbracket^{N-1}$.
\end{enumerate}

\begin{algorithm}[htbp]
	\caption{2-\SMIN$(\llbracket x \rrbracket, \llbracket y \rrbracket)\to \llbracket min(x, y) \rrbracket$}
        \label{alg:2-SMIN}
        \vspace{3pt}
	\KwIn{$S_1$ holds $\llbracket x \rrbracket$ and $\llbracket y \rrbracket$.}
        \KwOut{$S_1$ obtains $\llbracket min(x, y) \rrbracket$.}
        \vspace{4pt}

        \begin{mdframed}[backgroundcolor=cpcolor,innerleftmargin=2pt,innerrightmargin=2pt,innerbottommargin=2pt,leftmargin=-8pt,rightmargin=12pt]
            $\triangleright$ Step 1: $S_1$ computes
        \begin{mdframed}[backgroundcolor=whitecolor,innertopmargin=4pt,innerbottommargin=4pt,leftmargin=0pt,rightmargin=1pt,innerleftmargin=-4pt]
            \quad $\bullet$ $\llbracket D \rrbracket \gets (\llbracket x \rrbracket \cdot \llbracket y \rrbracket^{N - 1})^{r_1} \cdot \llbracket r_1 + r_2 \rrbracket$ if $\pi = 0$, \\
            \hangindent=7pt or $\llbracket D \rrbracket \gets (\llbracket y \rrbracket \cdot \llbracket x \rrbracket^{N - 1})^{r_1} \cdot \llbracket r_2 \rrbracket$ if $\pi = 1$, \\
            \hangindent=7pt where $r_1 \stackrel{\$}{\gets} \{0,1\}^\sigma $\textbackslash $\{0\}$, $r_2 \leq \frac{N}{2}$ and $r_1 + r_2 > \frac{N}{2}$;

            \quad $\bullet$ $\llbracket D_1 \rrbracket \gets$ \texttt{PDec}($\llbracket D \rrbracket$);

            \quad $\bullet$ and sends $\langle \llbracket D \rrbracket, \llbracket D_1 \rrbracket, \llbracket x \rrbracket, \llbracket y \rrbracket \rangle$ to $S_2$.
        \end{mdframed}
        \end{mdframed}
    
        \vspace{-5pt}

        \begin{mdframed}[backgroundcolor=cspcolor,innerleftmargin=2pt,innerrightmargin=2pt,innerbottommargin=2pt,leftmargin=-8pt,rightmargin=12pt]
             $\triangleright$ Step 2: $S_2$ computes
        \begin{mdframed}[backgroundcolor=whitecolor,innertopmargin=4pt,innerbottommargin=4pt,leftmargin=0pt,rightmargin=1pt,innerleftmargin=-4pt]
            \quad $\bullet$ $\llbracket D_2 \rrbracket \gets$ \texttt{PDec}($\llbracket D \rrbracket$) and $D \gets$ \texttt{TDec}($\llbracket D_1 \rrbracket, \llbracket D_2 \rrbracket$);

            \quad $\bullet$ $\llbracket d_0 \rrbracket \gets \llbracket y \rrbracket \cdot \llbracket 0 \rrbracket^r$ if $D > \frac{N}{2}$,
            
            \vspace{2pt}
            \hspace{4pt} or $\llbracket d_0 \rrbracket \gets \llbracket x \rrbracket \cdot \llbracket 0 \rrbracket^r$ if $D \leq \frac{N}{2}$, \\
            \hangindent=7pt where $r \stackrel{\$}{\gets} \{0, 1\}^\ell$;

            \quad $\bullet$ and sends $\llbracket d_0 \rrbracket$ to $S_1$.
    
       \end{mdframed}
        \end{mdframed}

        \vspace{-5pt}

        \begin{mdframed}[backgroundcolor=cpcolor,innerleftmargin=2pt,innerrightmargin=2pt,innerbottommargin=2pt,leftmargin=-8pt,rightmargin=12pt]
            $\triangleright$ Step 3: $S_1$ computes
        \begin{mdframed}[backgroundcolor=whitecolor,innertopmargin=4pt,innerbottommargin=4pt,leftmargin=0pt,rightmargin=1pt,innerleftmargin=-4pt]
            \quad $\bullet$ $\llbracket min(x, y) \rrbracket = \llbracket d_0 \rrbracket$ if $\pi = 0$, \\
            \hangindent=7pt otherwise $\llbracket min(x, y) \rrbracket = \llbracket x \rrbracket \cdot \llbracket y \rrbracket \cdot \llbracket d_0 \rrbracket^{N-1}$ if $\pi = 1$.
        \end{mdframed}
        \end{mdframed}


\end{algorithm}

As described in Algorithm \ref{alg:n-SMIN}, it is easy to construct $n$-\SMIN~based on the proposed 2-\SMIN. Formally, $n$-\SMIN~takes $n$ ciphertexts $\llbracket x_1 \rrbracket, \llbracket x_2 \rrbracket, \cdots, \llbracket x_n \rrbracket$ as inputs, and outputs $\llbracket min(x_1, x_2, \cdots, x_n) \rrbracket$.

\begin{algorithm}[htbp]
	\caption{$n$-\SMIN($\llbracket x_1 \rrbracket, \llbracket x_2 \rrbracket, \cdots, \llbracket x_n \rrbracket$) \\
                        \rightline{$\to \llbracket min(x_1, x_2, \cdots, x_n) \rrbracket$ \hspace{-20pt}}}
        \label{alg:n-SMIN}
        \vspace{3pt}
	\KwIn{$S_1$ holds $\llbracket x_1 \rrbracket, \llbracket x_2 \rrbracket, \cdots, \llbracket x_n \rrbracket$.}
        \KwOut{$S_1$ obtains $\llbracket min(x_1, x_2, \cdots, x_n) \rrbracket$.}
        \vspace{4pt}

        \begin{mdframed}[backgroundcolor=whitecolor,innerleftmargin=2pt,innerrightmargin=2pt,innerbottommargin=4pt,leftmargin=-8pt,rightmargin=12pt]
        $\llbracket min(x_1, x_2, \cdots, x_n) \rrbracket \gets \llbracket x_1 \rrbracket$;
        
        $t = 2$;
        
        \While{$t\leq n$} {
    
            $\llbracket min(x_1, x_2, \cdots, x_n) \rrbracket \gets$ 

            \rightline{2-\SMIN($\llbracket min(x_1, x_2, \cdots, x_n) \rrbracket,\llbracket x_t \rrbracket$);}
            
            $t = t + 1$;
        
        }

        \textbf{return} $\llbracket min(x_1, x_2, \cdots, x_n) \rrbracket$.
        \end{mdframed}
\end{algorithm}

\section{Pura Design}
\label{sec:solution}

In this section, we first formulate privacy-preserving face recognition. Next, we provide a brief overview of the proposed Pura. Following that, we elaborate on Pura.

\subsection{Problem Formulation}

Roughly speaking, privacy-preserving face recognition (PPFR) enables face recognition without leaking facial information. In this work, our proposed Pura fulfills PPFR through encrypting facial information, e.g., feature vectors.

Privacy-preserving face recognition can be formulated as:
\begin{equation}
    \label{eq:face_recognition}
    \Gamma(\llbracket p \rrbracket, \llbracket \mathcal{D} \rrbracket)=
    \begin{cases}
    \llbracket ID \rrbracket, &\Omega^\diamond(\llbracket p \rrbracket, \llbracket \mathcal{D} \rrbracket) \leq \llbracket \epsilon \rrbracket\\
    \perp, &\text{\textit{otherwise}}
    \end{cases},
\end{equation}
where $\llbracket p \rrbracket$ represents an encrypted feature vector to be recognized, $\llbracket \mathcal{D} \rrbracket$ refers to an encrypted feature vector database, and $\llbracket \epsilon \rrbracket$ denotes an encrypted threshold. PPFR takes input $\llbracket p \rrbracket$ and $\llbracket \mathcal{D} \rrbracket$ and securely calculates $\Omega^\diamond(\llbracket p \rrbracket, \llbracket \mathcal{D} \rrbracket)$ to obtain the minimal distance between $\llbracket p \rrbracket$ and all encrypted feature vectors in $\llbracket \mathcal{D} \rrbracket$. In this work, the distance metric is calculated by the square of Euclidean distance.

Note that a comparison operation over the ciphertexts is performed between a threshold $\llbracket \epsilon \rrbracket$ and the minimum distance between $\llbracket p \rrbracket$ and $\llbracket \mathcal{D} \rrbracket$. If $\Omega(p, \mathcal{D}) \leq \epsilon$, where $\Omega(p, \mathcal{D})$ means the minimum square of Euclidean distance between $p$ and $\mathcal{D}$, PPFR returns an encrypted identity, otherwise, returns $\perp$.

\subsection{Overview}

\begin{figure*}
	\centering
	\includegraphics[scale=0.54]{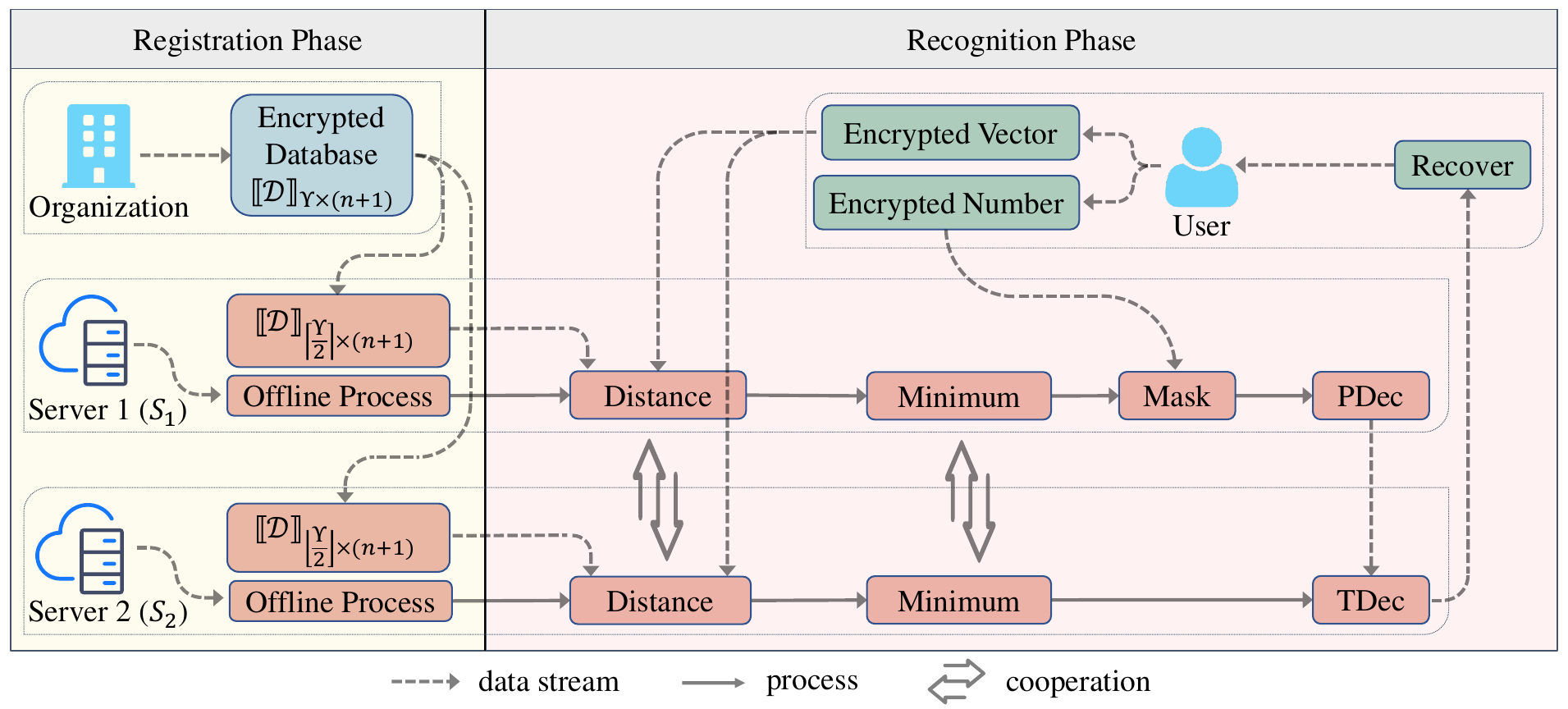}
	\caption{System overview of Pura.}
	\label{fig:main_workflow}
\end{figure*}

Our solution consists of two phases: \textbf{the registration phase} and \textbf{the recognition phase}. The workflow of our proposed Pura is sketched in Fig. \ref{fig:main_workflow}. 

During the registration phase, the organization encrypts a feature vector database and horizontally splits it into two parts, which are sent to $S_1$ and $S_2$, respectively. Additionally, each server performs offline operations in advance following the work \cite{zhao2023soci}. 

During the recognition phase, the user sends an encrypted feature vector to $S_1$ and $S_2$, along with an encrypted random number specifically to $S_1$. Then, based on the proposed protocols and the parallel mechanism, $S_1$ and $S_2$ parallelly and jointly accomplish the square of Euclidean distance and minimum process. After that, $S_1$ masks the result, $S_2$ decrypts and returns the masked result to the user. Finally, the user recovers and obtains the recognition result.

\subsection{Registration Phase}

\subsubsection{Data Transmission}
During the registration phase, the organization needs to encrypt the feature vector database and store the encrypted feature vector database on the twin servers. As shown in Fig. \ref{fig:main_workflow}, we divide the encrypted feature vector database into two parts and send them to two servers separately, instead of storing the encrypted feature vector database on a single server. Formally, the organization encrypts each value in the feature vector database $\mathcal{D}$, denoted as $\llbracket \mathcal{D} \rrbracket \gets$\texttt{Enc}($\mathcal{D}_{\Upsilon\times (n+1)}$) =
$$
\begin{pmatrix}
\llbracket ID_1 \rrbracket & \llbracket v_{1,1} \rrbracket & \llbracket v_{1,2} \rrbracket & \cdots & \llbracket v_{1,n} \rrbracket & \\
\llbracket ID_2 \rrbracket & \llbracket v_{2,1} \rrbracket & \llbracket v_{2,2} \rrbracket & \cdots & \llbracket v_{2,n} \rrbracket & \\
\vdots & \vdots & \vdots & \ddots & \vdots & \\
\llbracket ID_{\Upsilon} \rrbracket & \llbracket v_{\Upsilon,1} \rrbracket & \llbracket v_{\Upsilon,2} \rrbracket & \cdots & \llbracket v_{\Upsilon,n} \rrbracket &
\end{pmatrix}.
$$
In this work, the first $\lfloor \frac{\Upsilon}{2} \rfloor$ rows of $\llbracket \mathcal{D} \rrbracket$ are sent and stored on $S_2$, while the rest of $\llbracket \mathcal{D} \rrbracket$ are sent and stored on $S_1$. In this case, each server that holds half of the encrypted feature vector database can perform the protocols simultaneously, which fully takes advantage of the computation capability of each server.

\subsubsection{Offline Mechanism}
Inspired by the offline mechanism in the study \cite{zhao2023soci}, during the registration phase, we also utilize the offline mechanism. Specifically, $S_1$ and $S_2$ respectively generate $\llbracket r_1 \rrbracket$, $\llbracket r_2 \rrbracket$, $\llbracket -r_1r_2 \rrbracket$, $\llbracket \delta(r_1+r_2) \rrbracket$, $-r_1$, $-r_2$, $L$, $\llbracket r_1+r_2 \rrbracket$, $r$, $r_1$, $R$, and $\llbracket R \rrbracket$, which can be directly used in the recognition phase. Note that the quantity of random numbers generated is determined by the demands of the recognition phase.

\subsection{Recognition Phase}
\subsubsection{Square of Euclidean Distance}
During the recognition phase, the user encrypts a probe feature vector ($\llbracket p \rrbracket \gets$\texttt{Enc}($p$) = ($\llbracket u_1 \rrbracket$, $\llbracket u_2 \rrbracket$, $\cdots$, $\llbracket u_n \rrbracket$)) and then sends $\llbracket p \rrbracket$ to $S_1$ and $S_2$ at the same time. After obtaining the encrypted probe feature vector $\llbracket p \rrbracket$, $S_1$ and $S_2$ simultaneously compute the square of Euclidean distance between $\llbracket p \rrbracket$ and the item in $\llbracket \mathcal{D} \rrbracket$. Specifically, $S_1$ and $S_2$ first calculate the difference among feature vectors separately, which is formulated as
\begin{equation}
    \llbracket Q \rrbracket = 
    \begin{pmatrix}
        \llbracket u_1 \rrbracket\cdot\llbracket v_{1,1} \rrbracket^{N-1} & \cdots & \llbracket u_n \rrbracket\cdot\llbracket v_{1,n} \rrbracket^{N-1} & \\
        \vdots & \ddots & \vdots & \\
        \llbracket u_1 \rrbracket\cdot\llbracket v_{\phi,1} \rrbracket^{N-1} & \cdots & \llbracket u_n \rrbracket\cdot\llbracket v_{\phi,n} \rrbracket^{N-1} &
    \end{pmatrix},
\end{equation}
where $\phi$ is the rows of $\llbracket \mathcal{D} \rrbracket$ that each server holds.
After obtaining $\llbracket Q \rrbracket$, $S_1$ and $S_2$ parallelly and jointly perform \BatchSquare~to compute the square of every ciphertext in $\llbracket Q \rrbracket$. The result is denoted as a matrix $\llbracket E \rrbracket$ with $\phi$ rows and $n$ columns of items $\llbracket e \rrbracket$. Once they finish \BatchSquare, they can locally compute the square of Euclidean distances. According to the additive homomorphism, they compute the sum of each row in $\llbracket E \rrbracket$, which is defined as
\begin{equation}
\label{eq:square_dist}
    \llbracket D_k \rrbracket = (\llbracket d_1 \rrbracket, \cdots, \llbracket d_\phi \rrbracket) = (\prod_{j = 1}^{j = n} \llbracket e_{1, j}  \rrbracket, \cdots, \prod_{j = 1}^{j = n} \llbracket e_{\phi, j}  \rrbracket),
\end{equation}
where $S_1$ and $S_2$ obtain $\llbracket D_1 \rrbracket$ and $\llbracket D_2 \rrbracket$, respectively.

\subsubsection{Minimum Distance}
(i) $S_1$ and $S_2$ jointly call $n$-\SMIN~and take input $\llbracket D_1 \rrbracket$ and $\llbracket D_2 \rrbracket$, respectively. Note that $n$ of $n$-\SMIN~is the size of $\llbracket D_i \rrbracket$ for $S_k$, where $k \in \{1, 2\}$. After that, $S_1$ obtains the minimum result $\llbracket d_{S_1} \rrbracket$ among $\llbracket D_1 \rrbracket$, and $S_2$ gets $\llbracket d_{S_2} \rrbracket$ in the same way. (ii) $S_2$ sends $\llbracket d_{S_2} \rrbracket$ to $S_1$. Then $S_1$ inputs two minimum results $\llbracket d_{S_1} \rrbracket$ and $\llbracket d_{S_2} \rrbracket$, as well as an encrypted threshold $\llbracket \epsilon \rrbracket$, and calls an extra $n$-\SMIN~protocol with $S_2$ to obtain the encrypted recognition result $\llbracket \gamma \rrbracket$. Specifically, $n$ of the extra $n$-\SMIN~protocol is 3.

\subsubsection{Mask and Recover}
Upon obtaining $\llbracket \gamma \rrbracket$, $S_1$ masks $\llbracket \gamma \rrbracket$ with $\llbracket R \rrbracket$ via the additive homomorphism, i.e., $\llbracket \gamma + R \rrbracket$ = $\llbracket \gamma \rrbracket \cdot \llbracket R \rrbracket$, where $R \stackrel{\$}{\gets} \{0,1\}^\sigma$. Note that the random number $R$ is generated and encrypted as $\llbracket R \rrbracket$ by the user. $\llbracket R \rrbracket$ is sent to $S_1$ along with the encrypted feature vector. Next, $S_1$ calls \texttt{PDec} to partially decrypt $\llbracket \gamma + R \rrbracket$ and sends the ciphertexts to $S_2$. $S_2$ obtains $\gamma + R$ by calling \texttt{PDec} and \texttt{TDec}. Eventually, the user receives $\gamma + R$ from $S_2$, and easily recovers the recognition result $\gamma$.

\section{Privacy Analysis}
\label{sec:analysis}

In this section, we demonstrate that the proposed secure computing protocols used to construct Pura prevent users' facial information from being leaked.

Due to the semantically secure property of threshold Paillier cryptosystems \cite{lysyanskaya2001adaptive, liu2016privacy}, it has been widely adopted to protect users' images by performing operations in the encrypted domain. Zhao \textit{et al.} \cite{zhao2022soci} proved that $x+r$ is a chosen-plaintext attack (CPA) secure one-time key encryption scheme. According to the computational indistinguishability experiment between an adversary and a challenger presented in the work \cite{katz2007introduction}, we denote the computational indistinguishability experiment as PriR$_{\mathcal{A},x+r}^{cpa}$ and present a detailed description in the following.

\begin{enumerate}
    \item An adversary $\mathcal{A}$ randomly chooses two random numbers $x_0$, $x_1$ such that $x_0, x_1 \in [-2^\ell, 2^\ell]$;

    \item A challenger randomly selects $r \stackrel{\$}{\gets} \{0,1\}^\sigma$ and $b \stackrel{\$}{\gets} \{0,1\}$, where $\sigma$ is a secure parameter. The challenger then calculates $x_b + r$ and returns to $\mathcal{A}$.

    \item $\mathcal{A}$ yields a bit $b'$.

    \item The experiment outputs 1 if $b' = b$; otherwise, it outputs 0. In the case when the output is 1, PriR$_{\mathcal{A},x+r}^{cpa}$ equals to 1, which means $\mathcal{A}$ succeeds.
\end{enumerate}

\begin{theorem}
    \label{theo:x+r}
    For two random numbers $x_0, x_1 \in [-2^\ell, 2^\ell]$, $x_0 + r_0$ and $x_1 + r_1$ are considered computationally indistinguishable, where $r_0, r_1 \stackrel{\$}{\gets} \{0,1\}^\sigma$. More precisely, given two random numbers $x_b \in [-2^\ell, 2^\ell]$ and $r \stackrel{\$}{\gets} \{0, 1\}^\sigma$, the formula Pr$[b'=b|x_b+r] \leq \frac{1}{2}$ + negl($\sigma$) always holds, where negl($\sigma$) is a negligible function of $\sigma$ and $b, b' \stackrel{\$}{\gets} \{0, 1\}$. The study \cite{zhao2022soci} offers the proof details.
\end{theorem}

\begin{corollary}
\label{cor:batchsmul}
Assuming a pair of ciphertexts of length $s$: $\llbracket x_1 \rrbracket, \llbracket x_2 \rrbracket, \cdots, \llbracket x_s \rrbracket$ and $\llbracket y_1 \rrbracket, \llbracket y_2 \rrbracket, \cdots, \llbracket y_s \rrbracket$, where $x_i, y_i \in [-2^\ell, 2^\ell]$, the proposed \BatchSMUL~protocol securely computes $\llbracket x_1y_1 \rrbracket, \llbracket x_2y_2 \rrbracket, \cdots, \llbracket x_sy_s \rrbracket$ under non-colluding attackers $\mathcal{A} = (\mathcal{A}_{S1}, \mathcal{A}_{S2}$), where $\mathcal{A}_{S1}$ and $\mathcal{A}_{S2}$ denote $S_1$ and $S_2$ as polynomial-time adversaries, respectively. That means \BatchSMUL~does not leak $x_i$ and $y_i$ to $\mathcal{A}_{S1}$ and $\mathcal{A}_{S2}$.
\end{corollary}

\begin{proof}
\label{pro:batchsmul}
We now construct the independent simulators ($\mathcal{S}_{S1}, \mathcal{S}_{S2}$), which simulates $S_1$ and $S_2$, respectively.

$\mathcal{S}_{S1}$ simulates the view of $\mathcal{A}_{S1}$ in the real in the following:
\begin{enumerate}
    \item $\mathcal{S}_{S1}$ takes $\langle \llbracket x_i \rrbracket, \llbracket y_i \rrbracket, \llbracket (x_i+r_{i,1})\cdot(y_i+r_{i,2}) \rrbracket \rangle$ as inputs and then randomly chooses $\hat{x}_i, \hat{y}_i \in [-2^\ell, 2^\ell]$ and $\hat{r}_{i,1}, \hat{r}_{i,2}, \hat{\lambda}_1 \stackrel{\$}{\gets} \{0, 1\}^\sigma$.
    
    \item $\mathcal{S}_{S1}$ calls \texttt{Enc} to encrypt $\hat{x}_i, \hat{y}_i, \hat{x}_i+\hat{r}_{i,1}+\delta, \hat{y}_i+\hat{r}_{i,2}+\delta, -\hat{r}_{i,2}\cdot(\hat{x}_i+\delta), -\hat{r}_{i,1}\cdot(\hat{y}_i+\delta), -\hat{r}_{i,1}\cdot\hat{r}_{i,2}, \delta\cdot(\hat{r}_{i,1}+\hat{r}_{i,2})$ into $\llbracket \hat{x}_i \rrbracket, \llbracket \hat{y}_i \rrbracket, \llbracket \hat{X}_i \rrbracket, \llbracket \hat{Y}_i \rrbracket, \llbracket -\hat{r}_{i,2}\cdot(\hat{x}_i+\delta), \rrbracket, \llbracket -\hat{r}_{i,1}\cdot(\hat{y}_i+\delta) \rrbracket, \llbracket -\hat{r}_{i,1}\cdot\hat{r}_{i,2} \rrbracket, \llbracket \delta\cdot(\hat{r}_{i,1}+\hat{r}_{i,2}) \rrbracket$, respectively. Then, $\mathcal{S}_{S1}$ calculates $\llbracket \hat{C} \rrbracket = \prod_{i=1}^{i=s} \llbracket \hat{X}_i \rrbracket^{L^{2i-1}} \cdot \llbracket \hat{Y}_i \rrbracket^{L^{2i-2}}$, and $\llbracket \hat{C_1} \rrbracket = \llbracket \hat{C} \rrbracket^{\hat{\lambda}_1}$ by calling \texttt{PDec}, and $\llbracket \hat{x}_i\cdot\hat{y}_i \rrbracket = \llbracket (x_i+r_{i,1})\cdot(y_i+r_{i,2}) \rrbracket \cdot \llbracket -\hat{r}_{i,2}\cdot(\hat{x}_i+\delta) \rrbracket \cdot \llbracket -\hat{r}_{i,1}\cdot(\hat{y}_i+\delta) \rrbracket \cdot \llbracket -\hat{r}_{i,1}\cdot\hat{r}_{i,2} \rrbracket \cdot \llbracket \delta\cdot(\hat{r}_{i,1}+\hat{r}_{i,2}) \rrbracket$.
    
    \item $\mathcal{S}_{S1}$ yields $\langle \llbracket \hat{x}_i \rrbracket$, $\llbracket \hat{y}_i \rrbracket$, $\llbracket \hat{X}_i \rrbracket$, $\llbracket \hat{Y}_i \rrbracket$, $\llbracket \hat{C} \rrbracket$, $\llbracket \hat{C}_1 \rrbracket$, $\llbracket -\hat{r}_{i,2}\cdot(\hat{x}_i+\delta), \rrbracket$, $\llbracket -\hat{r}_{i,1}\cdot(\hat{y}_i+\delta) \rrbracket$, $\llbracket -\hat{r}_{i,1}\cdot\hat{r}_{i,2} \rrbracket$, and $\llbracket \delta\cdot(\hat{r}_{i,1}+\hat{r}_{i,2}) \rrbracket \rangle$ as the simulation of $\mathcal{A}_{S1}$'s entire view.
\end{enumerate}

$\mathcal{A}_{S1}$ only learns $\llbracket \hat{C} \rrbracket$ and $\llbracket \hat{C}_1 \rrbracket$. Since the (2,2)-threshold Paillier cryptosystem is semantically secure, $\mathcal{A}_{S_1}$ cannot infer $\llbracket \hat{x}_i \rrbracket, \llbracket \hat{y}_i \rrbracket, \llbracket \hat{x}_i \cdot \hat{y}_i \rrbracket$ from $\llbracket \hat{C} \rrbracket$ and $\llbracket \hat{C}_1 \rrbracket$. Thus, $\llbracket \hat{x}_i \rrbracket$ and $\llbracket x_i \rrbracket$, $\llbracket \hat{y}_i \rrbracket$ and $\llbracket y_i \rrbracket$, $\llbracket \hat{C} \rrbracket$ and $\llbracket C \rrbracket$, $\llbracket \hat{C}_1 \rrbracket$ and $\llbracket C_1 \rrbracket$ are computationally indistinguishable.

$\mathcal{S}_{S2}$ simulates the view of $\mathcal{A}_{S2}$ in the real in the following:
\begin{enumerate}
    \item $\mathcal{S}_{S2}$ takes $\langle \llbracket \sum_{i=s}^{i=1}L^{2i-1}\cdot(x_i+r_{i,1}+\delta)+L^{2i-2}\cdot(y_i+r_{i,2}+\delta) \rrbracket, \llbracket \sum_{i=s}^{i=1}L^{2i-1}\cdot(x_i+r_{i,1}+\delta)+L^{2i-2}\cdot(y_i+r_{i,2}+\delta) \rrbracket^{\lambda_1} \rangle$ as inputs and then randomly chooses $\bar{x}_i, \bar{y}_i\in [-2^\ell, 2^\ell]$ and $\bar{r}_{i,1}, \bar{r}_{i,2}, \bar{\lambda}_1, \bar{\lambda}_2 \stackrel{\$}{\gets} \{0, 1\}^\sigma$.

    \item $\mathcal{S}_{S2}$ calls \texttt{Enc} to encrypt $\bar{x}_i+\bar{r}_{i,1}, \bar{y}_i+\bar{r}_{i,2}, (\bar{x}_i+\bar{r}_{i,1})\cdot(\bar{y}_i+\bar{r}_{i,2})$ into $\llbracket \bar{X}_i \rrbracket, \llbracket \bar{Y}_i \rrbracket, \llbracket (\bar{x}_i+\bar{r}_{i,1})\cdot(\bar{y}_i+\bar{r}_{i,2}) \rrbracket$, respectively, and calculates $\llbracket\bar{X}_{i,1}\rrbracket=\llbracket\bar{X}_i\rrbracket^{\bar{\lambda}_1}, \llbracket\bar{X}_{i,2}\rrbracket=\llbracket\bar{X}_i\rrbracket^{\bar{\lambda}_2}, \llbracket\bar{Y}_{i,1}\rrbracket=\llbracket\bar{Y}_i\rrbracket^{\bar{\lambda}_1}, \llbracket\bar{Y}_{i,2}\rrbracket=\llbracket\bar{Y}_i\rrbracket^{\bar{\lambda}_2}$ by calling \texttt{PDec}.

    \item $\mathcal{S}_{S2}$ yields $\langle \llbracket\bar{X}_i\rrbracket, \llbracket\bar{X}_{i,1}\rrbracket, \llbracket\bar{X}_{i,2}\rrbracket, \llbracket\bar{Y}_i\rrbracket, \llbracket\bar{Y}_{i,1}\rrbracket, \llbracket\bar{Y}_{i,2}\rrbracket, \bar{x}_i+\bar{r}_{i,1}, \bar{y}_i+\bar{r}_{i,2}, (\bar{x}_i+\bar{r}_{i,1})\cdot(\bar{y}_i+\bar{r}_{i,2}), \llbracket (\bar{x}_i+\bar{r}_{i,1})\cdot(\bar{y}_i+\bar{r}_{i,2}) \rrbracket \rangle$ as the simulation of $\mathcal{A}_{S2}$'s entire view.
\end{enumerate}

In the view of $\mathcal{A}_{S2}$, $\mathcal{A}_{S2}$ only learns $(\bar{x}_i+\bar{r}_{i,1})\cdot(\bar{y}_i+\bar{r}_{i,2})$, so it fails to get $\bar{x}_i \cdot \bar{y}_i$. As the (2,2)-threshold Paillier cryptosystem is semantically secure, $\llbracket \bar{x}_i \rrbracket$ and $\llbracket x_i \rrbracket$, $\llbracket \bar{y}_i \rrbracket$ and $\llbracket y_i \rrbracket$ are computationally indistinguishable, demonstrating that $\mathcal{S}_{S2}$ in an ideal execution and $\mathcal{A}_{S2}$ in the real world are computationally indistinguishable.

Generally speaking, \BatchSMUL~is secure to calculate $\llbracket x_i \cdot y_i \rrbracket$ and does not leak $x_i$ and $y_i$ to $\mathcal{A}_{S1}$ and $\mathcal{A}_{S2}$.

\end{proof}

\begin{corollary}
\label{cor:batchsquare}
    Taking $s$ ciphertexts $\llbracket x_1 \rrbracket, \llbracket x_2 \rrbracket, \cdots, \llbracket x_s \rrbracket$, where $-2^\ell \leq x_i \leq 2^\ell$, the proposed \BatchSquare~protocol securely calculates $\llbracket x_1^2 \rrbracket, \llbracket x_2^2 \rrbracket, \cdots, \llbracket x_s^2 \rrbracket$ under non-colluding attackers $\mathcal{A} = (\mathcal{A}_{S1}, \mathcal{A}_{S2}$).
\end{corollary}

\begin{proof}
    Corollary \ref{cor:batchsquare} can easily be proved following the proof method of Corollary \ref{cor:batchsmul}.
\end{proof}

\begin{theorem}
    \label{theo:half}
    For two random numbers $x_0, x_1$ such that $x_0, x_1 \in [-2^\ell, 2^\ell]$, given $d = r_1 \cdot (x-y+1) + r_2$ or $d = r_1 \cdot (y-x) + r_2$, where $r_1, r_2$ are two positive random numbers such that $r_1 \stackrel{\$}{\gets} \{0, 1\}^\sigma$\textbackslash $\{0\}, r_1 + r_2 > \frac{N}{2}$ and $r_2 \leq \frac{N}{2}$, if Pr$[d=r_1\cdot(x-y+1)+r_2]$ = Pr$[d=r_1\cdot(y-x)+r_2]$ = $\frac{1}{2}$, the probability of successfully comparing the relative size of $x$ and $y$ is $\frac{1}{2}$. More formally, for random number $x, y \in [-2^\ell, 2^\ell]$, Pr$[d>\frac{N}{2}]$ = Pr$[d \leq \frac{N}{2}]$ = $\frac{1}{2}$. Detailed proof can be found in the study \cite{zhao2022soci}.
\end{theorem}

\begin{corollary}
\label{cor:2-smin}
Given two ciphertexts $\llbracket x \rrbracket$ and $\llbracket y \rrbracket$ such that $x, y \in [-2^\ell, 2^\ell]$, the proposed 2-\SMIN~protocol securely computes $\llbracket min(x,y) \rrbracket$ under non-colluding attackers $\mathcal{A} = (\mathcal{A}_{S1}, \mathcal{A}_{S2}$).
\end{corollary}

\begin{proof}
\label{pro:2-smin}
Following the proof method of Corollary \ref{cor:batchsmul}, Corollary \ref{cor:2-smin} is proved in the following.
\begin{enumerate}
    \item In the view of $\mathcal{A}_{S1}$, $\mathcal{A}_{S1}$ only learns the ciphertexts $\llbracket x \rrbracket, \llbracket y \rrbracket$, and $\llbracket r_1(x-y)+(r_1+r_2) \rrbracket$ or $\llbracket r_1(y-x)+r_2 \rrbracket$, so it fails to obtain $x, y$, and $min(x, y)$.

    \item In the view of $\mathcal{A}_{S2}$, $\mathcal{A}_{S2}$ can learn $r_1(x-y)+(r_1+r_2)$ or $r_1(y-x)+r_2$. However, according to Theorem \ref{theo:x+r} and Theorem \ref{theo:half}, $\mathcal{A}_{S2}$ fails to obtain $x, y$ and $min(x,y)$ from $r_1(x-y)+(r_1+r_2)$ or $r_1(y-x)+r_2$.
\end{enumerate}

In conclusion, 2-\SMIN~securely computes $\llbracket min(x, y) \rrbracket$ and does not leak $x$ and $y$ to $\mathcal{A}_{S1}$ and $\mathcal{A}_{S2}$.
\end{proof}

\begin{corollary}
\label{cor:n-smin}
    Given $n$ ciphertexts $\llbracket x_1 \rrbracket, \llbracket x_2 \rrbracket, \cdots, \llbracket x_n \rrbracket$, where $-2^\ell \leq x_i \leq 2^\ell$, the proposed $n$-\SMIN~protocol securely calculates $\llbracket min(x_1, x_2, \cdots, x_n) \rrbracket$ in a semi-honest (non-colluding) model.
\end{corollary}

\begin{proof}
    Based on Corollary \ref{cor:2-smin}, it is easy to derive Corollary \ref{cor:n-smin}, since $n$-\SMIN~is simply calling 2-\SMIN~repeatedly.
\end{proof}

\begin{theorem}
    \label{the:safe}
    Given an encrypted feature vector $\llbracket p \rrbracket$ and an encrypted feature vector database $\llbracket \mathcal{D} \rrbracket$, $S_1$ and $S_2$ cannot obtain the raw data $p$ and $\mathcal{D}$.
\end{theorem}

\begin{proof}
    \label{pro:safe}
    The proposed Pura consists of the following operations: compute the difference between an encrypted input feature vector and an encrypted feature vector database via homomorphic subtraction; calculate the square of every difference by calling \BatchSquare; compute the square of Euclidean distance through homomorphic addition; search for the minimum distance and compare it with a threshold by calling $n$-\SMIN.
    
    According to the threshold Paillier cryptosystem \cite{lysyanskaya2001adaptive}, the addition and subtraction operations over ciphertexts do not leak plaintext data. According to Corollary \ref{cor:batchsquare}, $S_1$ and $S_2$ can safely cooperate to compute the square of a pair of ciphertexts by calling \BatchSquare. Based on Corollary \ref{cor:n-smin}, $n$-\SMIN~can perform the minimum operation among a group of ciphertexts without leaking privacy. Lastly, the encrypted recognition result is masked by $S_1$ and the user recovers the real result. Based on the assumption that $S_1$ and $S_2$ do not collude with each other, neither $S_1$ nor $S_2$ is able to learn the plaintext of the recognition result, which can only be obtained by the user.
\end{proof}

In conclusion, our proposed Pura solution can securely perform the PPFR task without revealing the original data to $\mathcal{A}_{S2}$ and $\mathcal{A}_{S2}$. In other words, Pura protects users' facial information and avoids private data leakage.

Finally, it can be easily demonstrated that Pura satisfies the three security requirements of ISO/IEC 24745 \cite{iso2011iso}: irreversibility, unlinkability, and confidentiality. First, the Paillier homomorphic cryptosystem ensures irreversibility, as attackers cannot retrieve the original plaintext data from leaked encrypted data. Second, the unlinkability is guaranteed by using different public/private key pairs for distinct face recognition systems, preventing the association of samples from the same user across different systems. Third, according to Theorem \ref{the:safe}, our scheme meets the confidentiality requirement of ISO/IEC 24745, ensuring that neither external adversaries nor the internal servers can obtain users' facial data.

\section{Experiments}
\label{sec:experiments}

In this section, we evaluate the effectiveness and efficiency of our proposed Pura. 

\subsection{Experimental setting}

We adopt the popular Labeled Faces in the Wild (LFW) \cite{huang2008labeled} with 13,233 images of 5,749 persons and CASIA-WebFace (WebFace) \cite{yi2014learning} with 459,737 images of 10,575 persons as our experimental datasets. We transform each image into a 512-dimensional feature vector in advance using a pre-trained deep-learning model Facenet \cite{schroff2015facenet}. Each value of the feature vector is mapped into the range of $[0,1]$. Since the values of the feature vectors are real numbers, we transform the real numbers into integers by multiplying a constant number 10,000. Our scheme is implemented in C++, and the experiments are conducted on two servers, each of which is equipped with an AMD EPYC 7402 CPU and 128 GB of RAM.

Our proposed secure computing protocols are based on SOCI$^+$ \cite{zhao2023soci}. We also implement the protocols using SOCI \cite{zhao2022soci}, called Naive. We conduct the experiments setting $N$ as 1024. We compare Pura with two related solutions \cite{drozdowski2019application} and \cite{huang2023efficient} based on BFV \cite{brakerski2014leveled}. Note that the work \cite{drozdowski2019application} adopts a row-wise strategy and the scheme \cite{huang2023efficient} utilizes a column-wise strategy. We conduct the experiments at an 80-bit security level. We set $prime\_modulus=265,774,897$ and $hensel\_lifting=1$. The parameter $slots$ is to control the number of plaintext encrypted in a single ciphertext. We select $slots=520$ for the work \cite{drozdowski2019application}, and $slots=1,040$ for the scheme \cite{huang2023efficient}. Other parameters are shown in Table \ref{tab:slots}.

\begin{table}[ht]
\caption{Parameter settings of different $slots$.}
\label{tab:slots}
\centering
\begin{tabular}{ccc}
\toprule
$slots$ & $cyclotomic\_polynomial$ & $modulus\_chain\_bits$ \\ \midrule
520     & 4,847                    & 119                    \\
1040    & 11,135                   & 230                    \\ \bottomrule
\end{tabular}
\end{table}

\subsection{Effectiveness}

To evaluate the effectiveness of our solution, we compare it with a baseline face recognition algorithm, which is formulated in Section \ref{pre:fr}. We randomly choose 200 images of 10 people in LFW and 300 images of 10 people in WebFace. All images are transformed into 512-dimensional feature vectors in advance. For each vector, we exclude it from the feature vector database and select it as the probe feature vector to be identified. We compare the precision and recall between baseline and Pura, where the precision and recall are formulated by $Precision = \frac{TP}{TP + FP}$ and $Recall = \frac{TP}{TP + FN}$, where $TP$, $FP$ and $FN$ represent $true$ $positive$, $false$ $positive$ and $false$ $negative$, respectively. As the comparison result shown in Fig. \ref{fig:effectiveness}, it is obvious that Pura has the same precision and recall with baseline, which demonstrates the effectiveness of our proposed Pura.

\begin{figure}[ht]
    \centering
    \subfigure[LFW]{\includegraphics[width=.49\columnwidth]{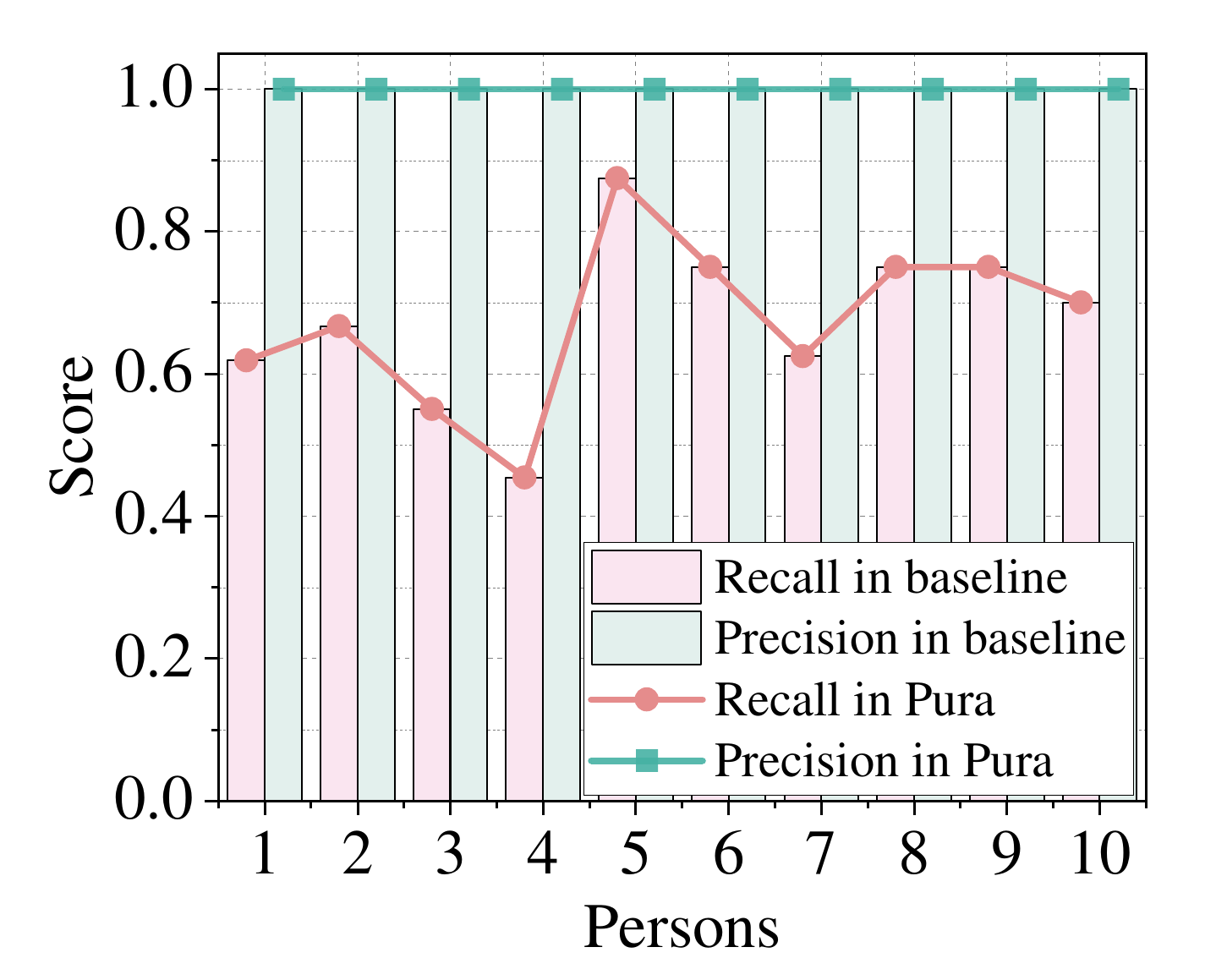}\label{fig:effectivenessLFW}}\hspace{1pt}
    \subfigure[WebFace]{\includegraphics[width=.49\columnwidth]{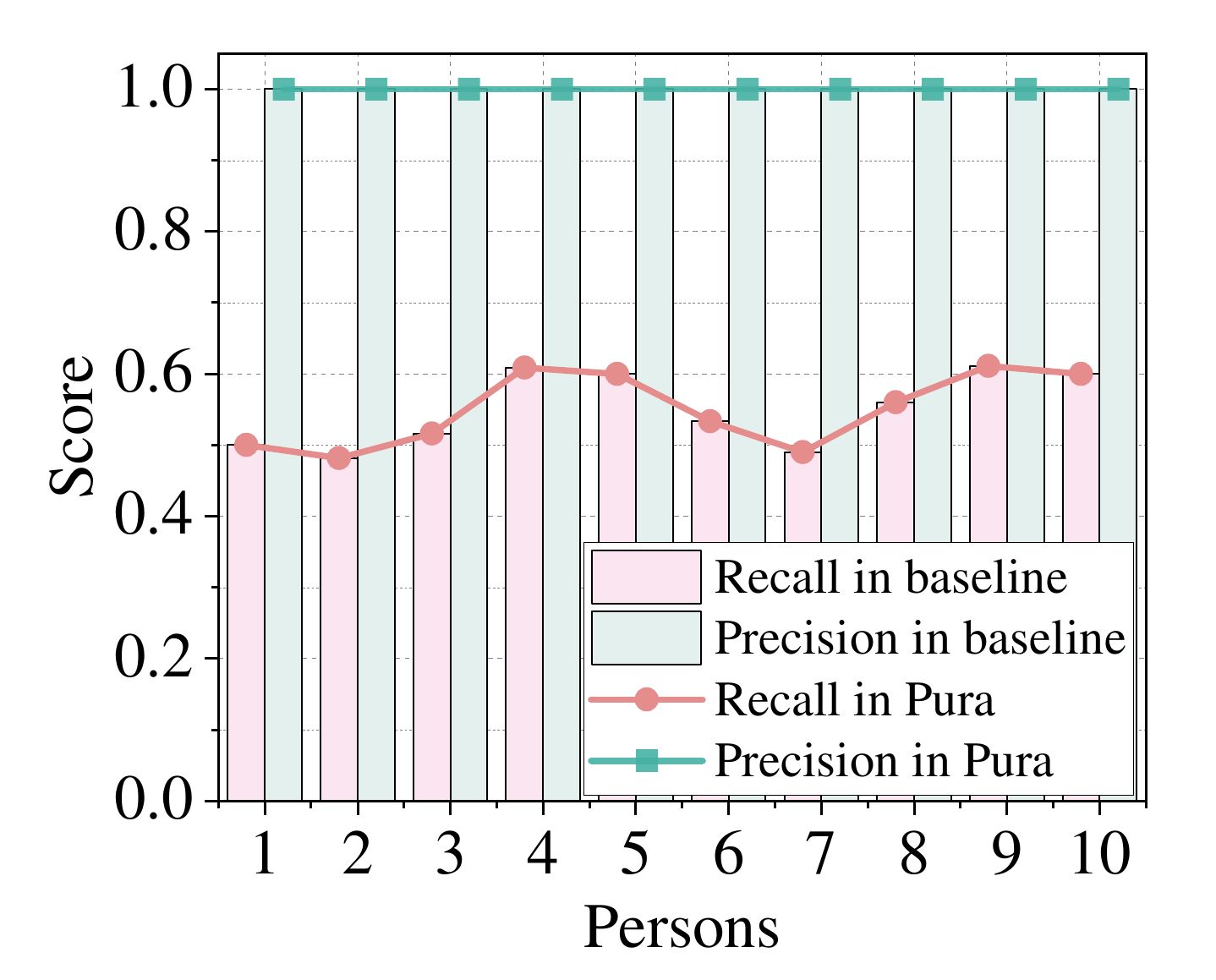}}\label{fig:effectivenessWebFace}
    \caption{Feasibility evaluation on LFW and WebFace.}
    \label{fig:effectiveness}
\end{figure}

We conduct a statistical analysis employing the Wilcoxon rank-sum test at a significance level of 0.05. When the $p$-value of the Wilcoxon rank-sum test is less than 0.05, it indicates that there are significant differences between the two solutions. Conversely, if the $p$-value exceeds 0.05, the test indicates no significant difference between the two solutions. As seen from Table \ref{tab:p-value}, the statistical $p$-value between Pura and baseline is greater than 0.05 in terms of the sum, mean, and variance. Therefore, we conclude that there is no significant difference between Pura and baseline when performing face recognition. Hence, we conclude that Pura is a feasible privacy-preserving face recognition solution.

\begin{table}[ht]
\caption{Statistical analysis between Pura and baseline.}
\label{tab:p-value}
\centering
\begin{tabular}{cccc}
\toprule
Dataset & Sum & Mean & Variance \\ \midrule
LFW     & 1   & 1    & 1        \\
WebFace & 1   & 1    & 1        \\ \bottomrule
\end{tabular}
\end{table}

\begin{figure}[ht]
    \centering
    \subfigure[Pura]{\includegraphics[width=.49\columnwidth,height=110pt]{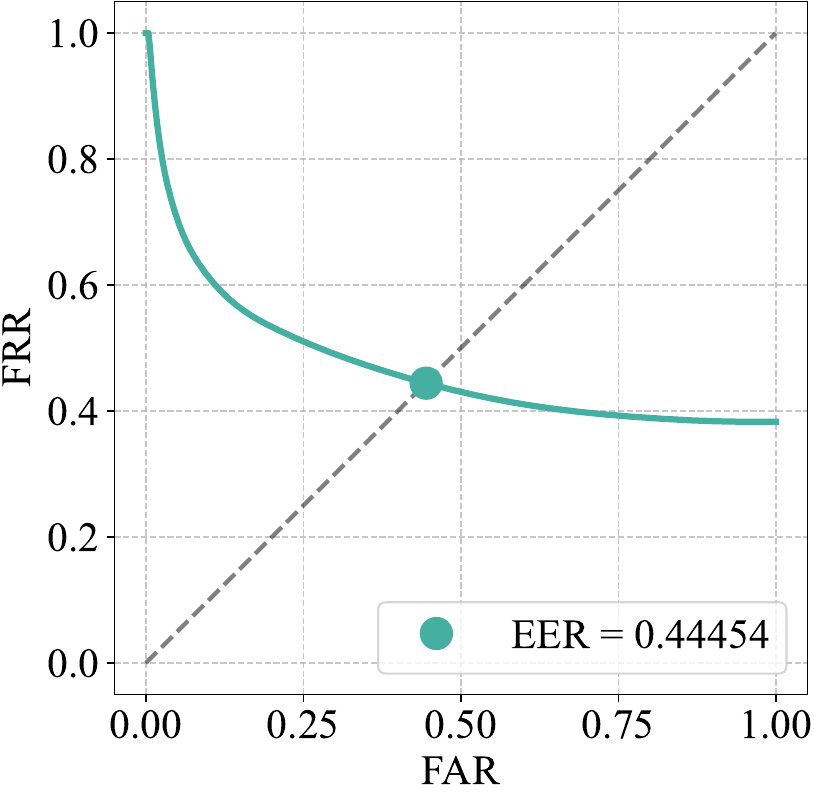}\label{fig:eer1}}\hspace{1pt}
    \subfigure[Basseline]{\includegraphics[width=.49\columnwidth,height=110pt]{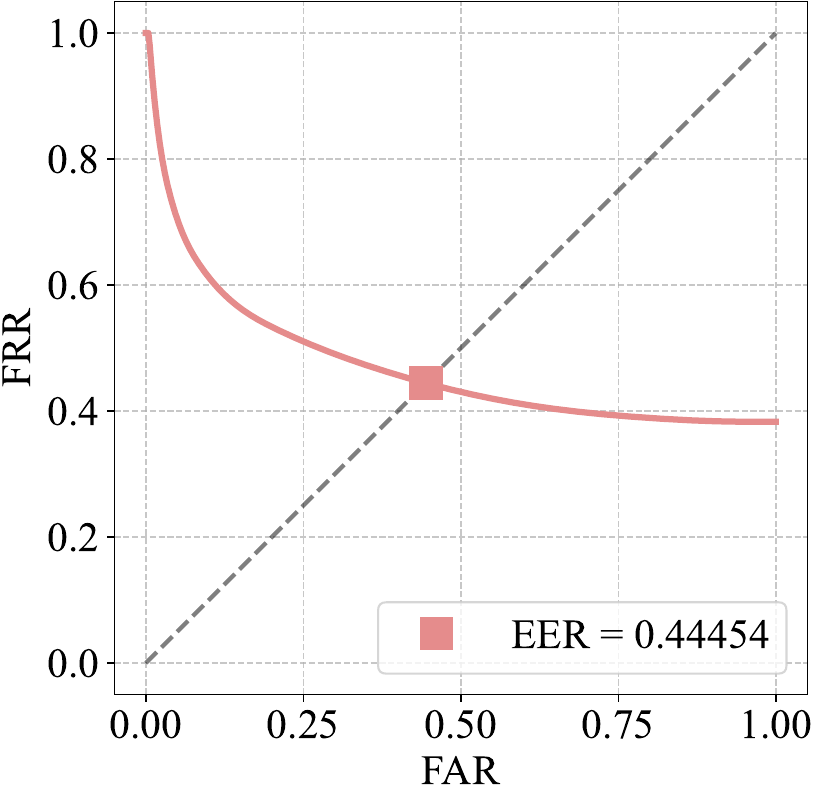}}\label{fig:eer2}
    \caption{DET curve comparison of Pura with baseline.}
    \label{fig:err}
\end{figure}

We further compare the Detection Error Tradeoff (DET) curves of Pura and baseline. The Equal Error Rate (EER) is the point where the False Acceptance Rate (FAR) equals the False Rejection Rate (FRR), presenting the performance of the face recognition system. As shown in Fig. \ref{fig:err}, the EER for Pura is 0.44454, which matches the EER for baseline. This result indicates that the proposed encryption scheme has no affect on the performance of the face recognition system, further demonstrating the effectiveness of Pura.

\subsection{Efficiency}

We first compare the encrypted feature vector database storage. We compare Pura with the work \cite{drozdowski2019application} and the scheme \cite{huang2023efficient}. The result is depicted in Fig. \ref{fig:db_storage}. It can be observed that BFV-based schemes suffer from high storage overhead. In terms of the storage overhead of a single server, the solution \cite{huang2023efficient} consumes more than 6 GB, while the solution \cite{drozdowski2019application} costs over 3 GB when the dataset size is 10,000. In contrast to the solutions \cite{drozdowski2019application} and \cite{huang2023efficient}, our proposed Pura consumes around 1 GB of storage overhead on a single server. Consequently, we can say that the proposed Pura enjoys less storage cost.

\begin{figure}[ht]
	\centering
	\includegraphics[scale=0.35]{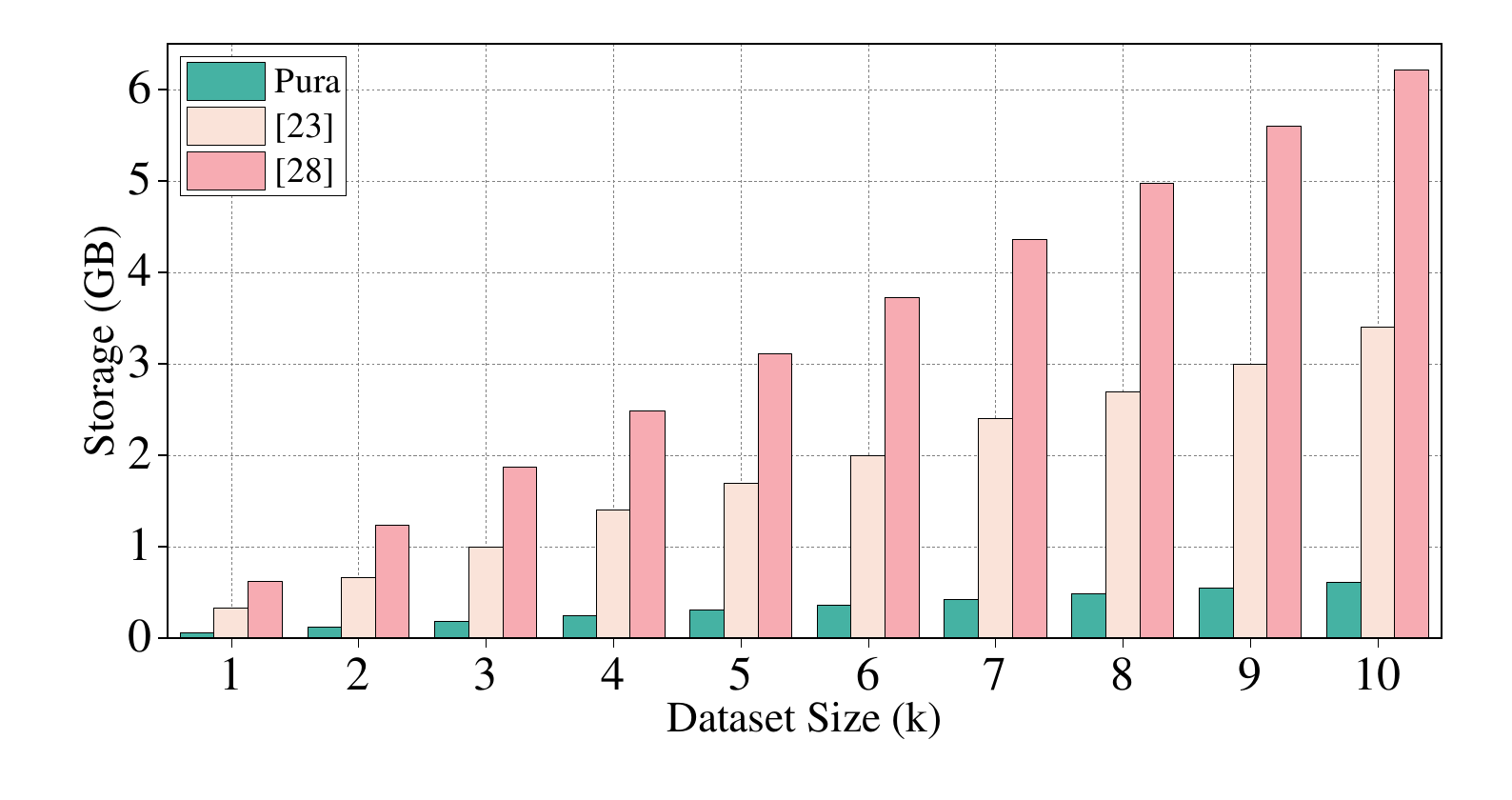}
	\caption{Comparison of storage costs of a single server among three solutions.}
	\label{fig:db_storage}
\end{figure}

Both solutions \cite{drozdowski2019application} and \cite{huang2023efficient} adopt BFV to perform the square operation. In order to demonstrate the efficiency of the proposed \BatchSquare~protocol, we compare the runtime of \BatchSquare~and the BFV-based scheme, as well as \BatchSMUL, in terms of square operations. We use the HElib library \cite{halevi2014algorithms} to implement the BFV square operation, named \texttt{BFV-Square}. We range the data size from 1,000 to 10,000 to compare the runtime. As shown in Table \ref{tab:comparison_square}, \BatchSquare~takes less time than \texttt{BFV-Square} generally. Note that in the square operation, \BatchSquare~can reduce the redundancies in \BatchSMUL, thereby decreasing the number of computed expressions and increasing the number of batch computations. Consequently, \BatchSquare~outperforms \BatchSMUL~in the square operation. Additionally, the growth rates of the runtime for \texttt{BFV-Square}, \BatchSMUL, and \BatchSquare~are 55.47, 41.15, and 25.98, respectively. Therefore, as the size of the dataset grows, \BatchSquare~shows significant advantage in terms of runtime.

\begin{table*}[ht]
\caption{Runtime of the square operation.}
\label{tab:comparison_square}
\centering
\begin{tabular}{ccccccccccc}
\toprule
Dataset Size (k)      & 1   & 2   & 3   & 4   & 5   & 6  & 7 & 8 & 9 & 10   \\ \midrule
\texttt{BFV-Square}                   & 57 ms  & 112 ms & 170 ms & 226 ms & 279 ms & 333 ms & 391 ms & 446 ms & 501 ms & 557 ms  \\
\BatchSMUL    & 104 ms & 156 ms & 191 ms & 238 ms & 274 ms & 321 ms & 354 ms & 395 ms & 439 ms & 481 ms  \\
\BatchSquare  & 60 ms  & 90 ms  & 125 ms & 153 ms & 173 ms & 202 ms & 224 ms & 240 ms & 272 ms & 304 ms  \\ \bottomrule
\end{tabular}
\end{table*}

Since the garbled circuit is one of the most popular approaches to finding the minimum value securely, we extensively compare our proposed $n$-\SMIN~scheme with the garbled circuit. We set the same security level (i.e., 80-bit security) in this experiment. Following the approach of constructing a garbled circuit to obtain the minimum value in the work \cite{huang2023efficient}, we separate each value $x$ into two shares $x_1$ and $x_2$, such that $x = x_1 + x_2$, and store $x_1$ and $x_2$ on different servers. In the $n$-\SMIN~scheme, each value is encrypted into ciphertext, and all ciphertexts are stored separately on different servers. As shown in Fig. \ref{fig:SMIN_GC}, $n$-\SMIN~requires less runtime and communication cost to find the minimum value than the garbled circuit. All in all, $n$-\SMIN~significantly outperforms the traditional garbled circuit.

\begin{figure}[ht]
    \centering
    \subfigure[Runtime]{\includegraphics[width=.48\columnwidth]{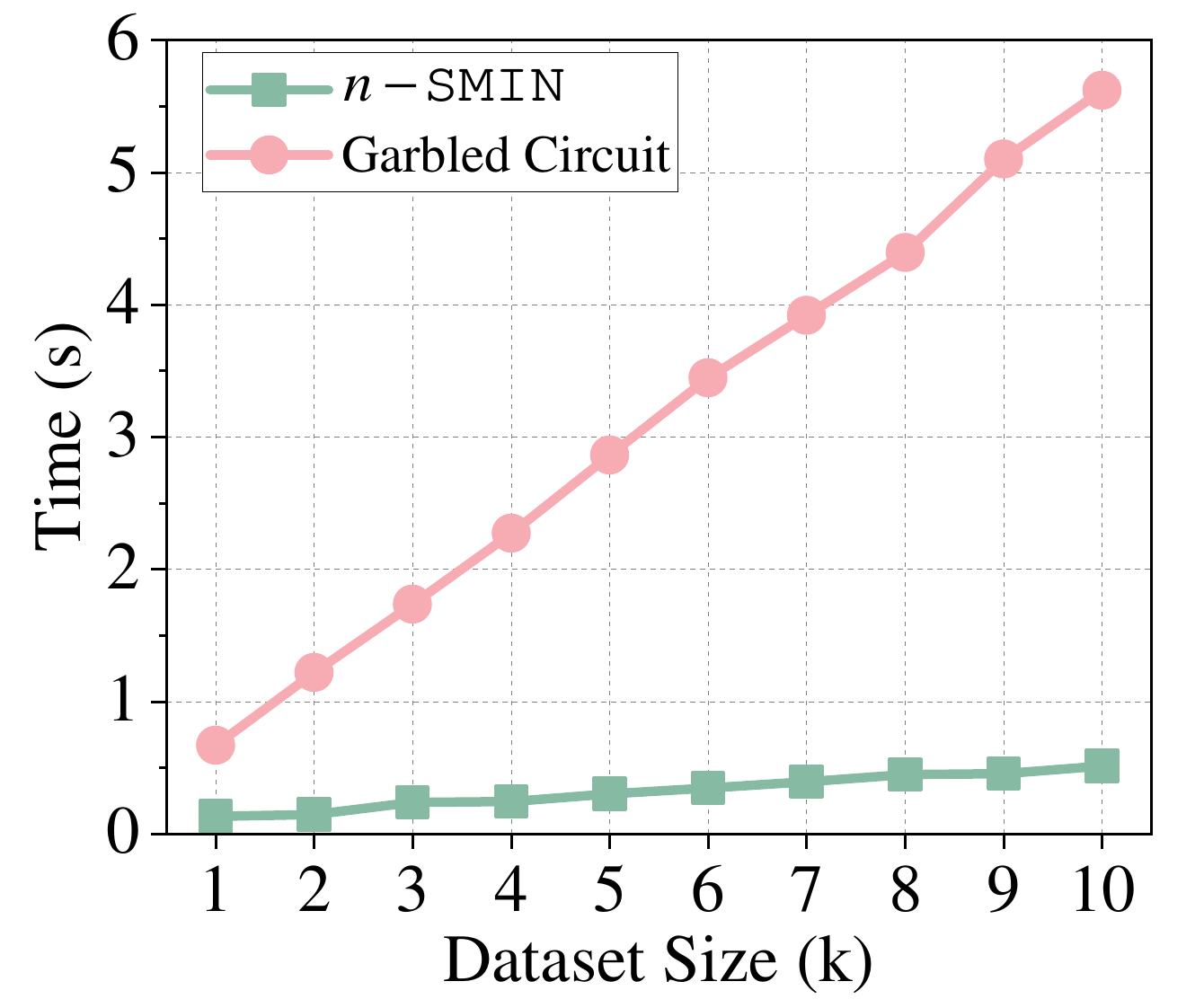}\label{fig:distance_time}}\hspace{5pt}
    \subfigure[Communication cost]{\includegraphics[width=.483\columnwidth]{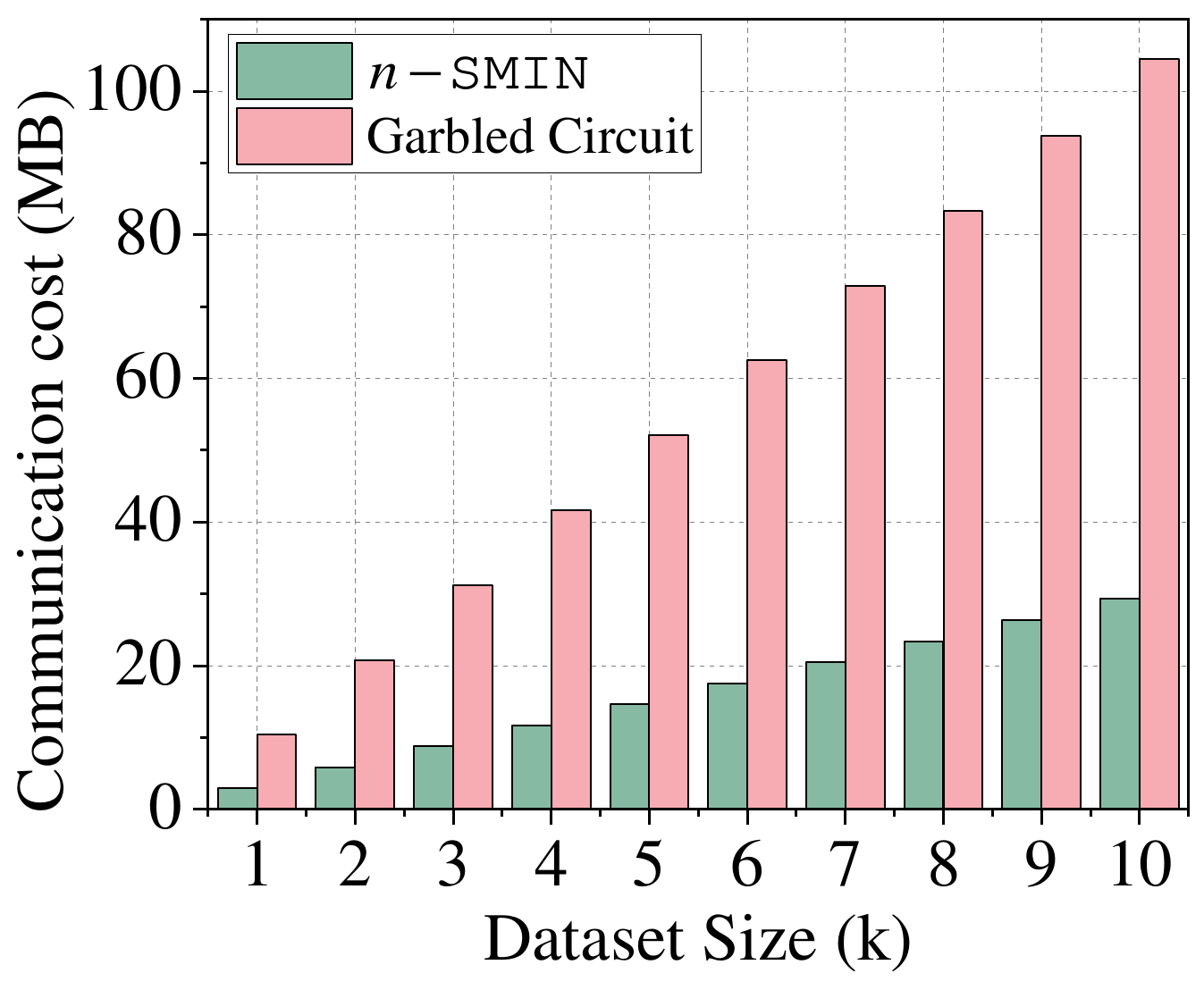}}\label{fig:distance_cost}
    \caption{Comparison between $n$-\SMIN~and garbled circuit.}
    \label{fig:SMIN_GC}
\end{figure}

\begin{figure}[ht]
    \centering
    \subfigure[Feature vector encryption]{\includegraphics[width=.478\columnwidth]
    {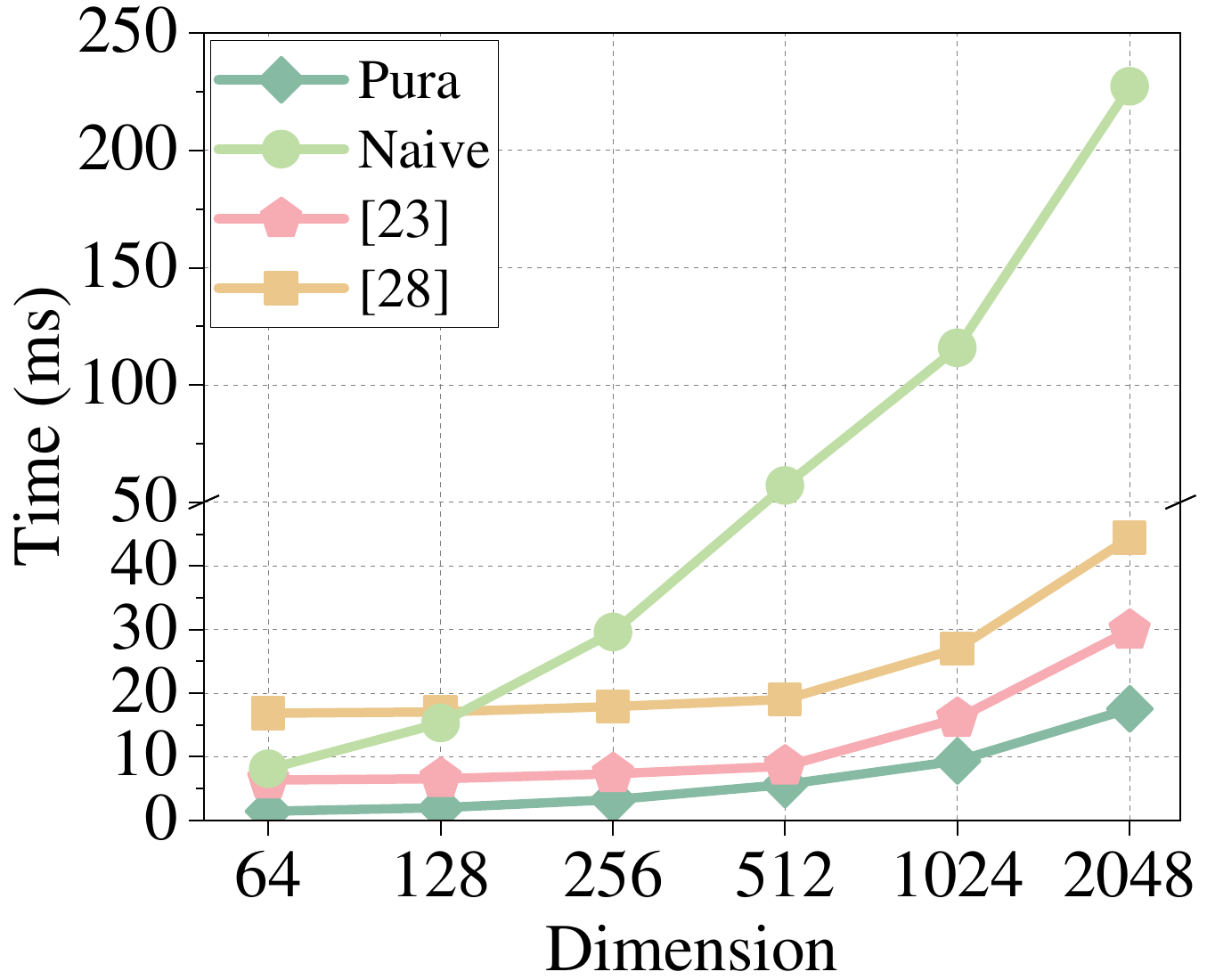}\label{sub_fig:feature_enc}}\hspace{5pt}
    \subfigure[Face recognition]{\includegraphics[width=.488\columnwidth]
    {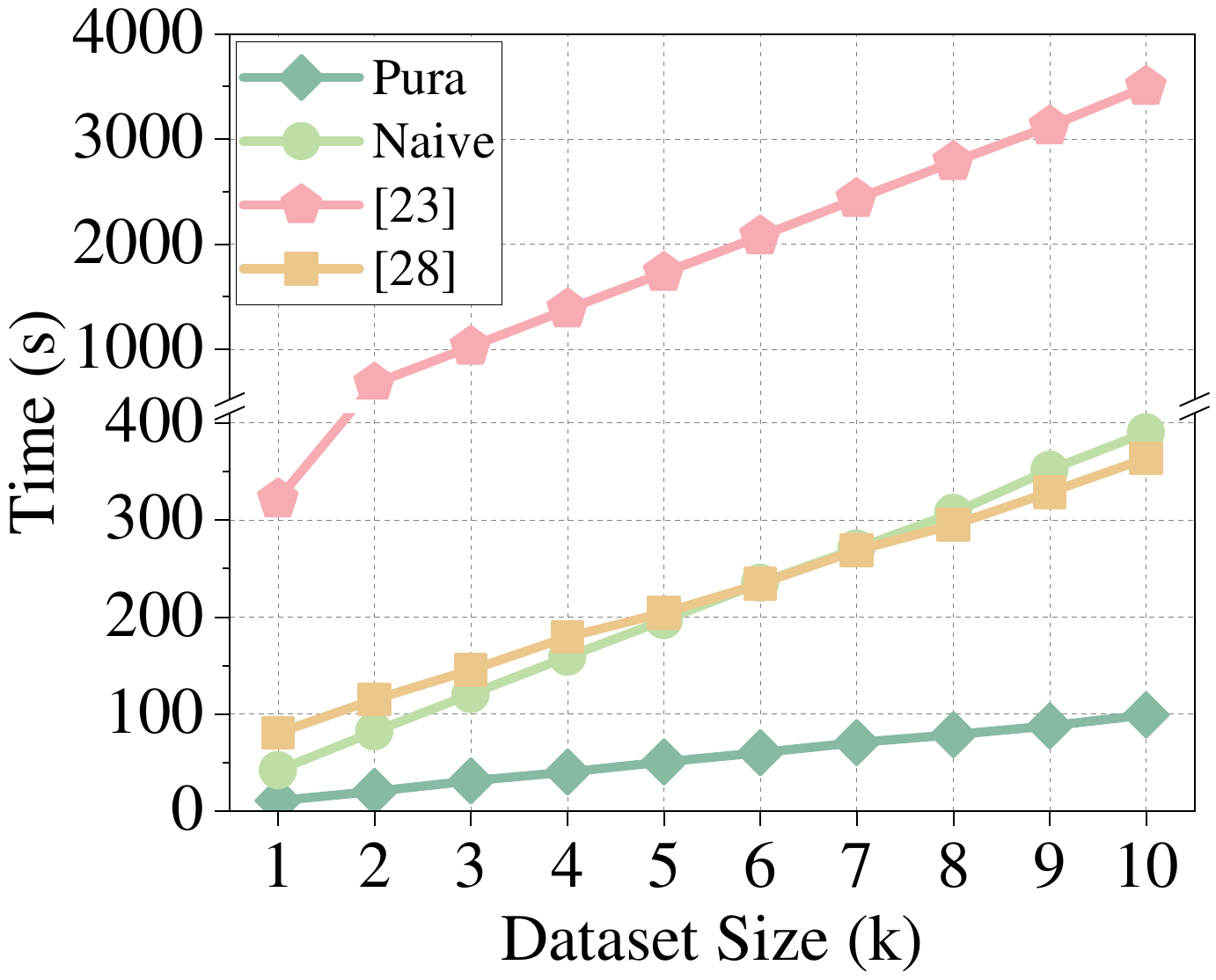}\label{sub_fig:recog}}
    \caption{Runtime comparison among four solutions.}
    \label{fig:recog}
\end{figure}

We evaluate the efficiency of computation for Pura through runtime. Fig. \ref{fig:recog} shows the runtime comparison between Pura, Naive, solutions \cite{drozdowski2019application} and \cite{huang2023efficient}. Fig. \ref{sub_fig:feature_enc} shows the runtime of encrypting a feature vector with varying dimensions. It can be observed that our proposed Pura outperforms the other schemes. Among the four schemes, Pura shows the slowest increase in runtime as the dimension of the feature vector changes. From Fig. \ref{sub_fig:recog}, we see that Pura is significantly faster than the other algorithms. In contrast to the work \cite{huang2023efficient}, Pura is 16 times faster than the work \cite{huang2023efficient} when the dataset size is 1,000. And the runtime of Pura is twice that of the work \cite{huang2023efficient} when the dataset size is 10,000. Since the work \cite{drozdowski2019application} utilizes a high-complexity strategy and excessive use of expensive homomorphic operations such as circular shifts, the recognition runtime is unacceptable. Thanks to the offline mechanism and the optimization of SOCI \cite{zhao2022soci}, Pura can save about 75\% of runtime compared to Naive.

Table \ref{tab:communication_total} shows the comparison results in terms of communication cost among the three solutions. As indicated in Table \ref{tab:communication_total}, the communication cost of Pura is less than that of the work \cite{huang2023efficient} in general. It can be observed that when the dataset size exceeds 9,000, the communication cost of Pura surpasses that of the work \cite{huang2023efficient}. We consider that this trade-off between communication cost and runtime in our scheme is acceptable, according to Fig. \ref{sub_fig:recog}. More crucially, the user experience is of greater importance for a face recognition system. Additionally, combining Table \ref{tab:communication_total} and Fig. \ref{sub_fig:recog}, although the communication cost of scheme \cite{drozdowski2019application} is the lowest, the sacrificed computational efficiency is unacceptable. The scheme requires a single server to perform all the complex and costly computations. While this reduces the communication volume, it results in an intolerable recognition latency. Additionally, as the scheme allows the server to decrypt and access the plaintext, its security level is significantly lower than that of our proposed Pura.

To summarize, our proposed Pura excels in runtime efficiency, significantly outperforming the other schemes. Also, Pura achieves a good balance between communication cost and runtime.

\begin{table*}[ht]
    \caption{Comparison of communication cost (GB) among three solutions.}
    \label{tab:communication_total}
    \centering
    \begin{tabular}{ccccccccccc}
    \toprule
    Dataset Size (k) & 1      & 2      & 3      & 4      & 5      & 6      & 7      & 8      & 9      & 10     \\   \midrule
    Pura             & 0.2061 & 0.4121 & 0.6172 & 0.8229 & 1.0295 & 1.2344 & 1.4404 & 1.6463 & 1.8516 & 2.0572 \\
    \cite{drozdowski2019application}         & 0.0012 & 0.0021 & 0.0030 & 0.0038 & 0.0047 & 0.0056 & 0.0064 & 0.0073 & 0.0084 & 0.0090 \\
    \cite{huang2023efficient}         & 1.6777 & 1.6872 & 1.6966 & 1.7062 & 1.7158 & 1.7251 & 1.7347 & 1.7442 & 1.7538 & 1.7631  \\ \bottomrule
    \end{tabular}
\end{table*}

\section{Conclusion}
\label{sec:conclusion}

In this paper, we propose Pura, an efficient privacy-preserving face recognition solution, which enables the twin servers to efficiently identify a user without leaking privacy. To achieve non-interactive and privacy-preserving face recognition, we propose a novel PPFR framework based on the threshold Paillier cryptosystem. We carefully design several secure underlying computing protocols to support secure and efficient operations over ciphertexts, including multiplication, square, and minimum. Moreover, we utilize a parallel computing mechanism to improve the efficiency of privacy-preserving face recognition. Our proposed Pura avoids degrading recognition accuracy while achieving remarkable performance.
The privacy analysis details the security of our design through the proposed underlying protocols. The experimental results demonstrate that our solution outperforms the state-of-the-art in terms of communication cost and runtime. For future work, we intend to deploy our proposed Pura into a real product.

\section{Acknowledgments}

The authors would like to thank the Editor-in-Chief, the Associate Editor, and the reviewers for their valuable comments and suggestions. This work is partially supported by the National Natural Science Foundation of China (No. 62202358), the China PostDoctoral Science Foundation (No. 2023TQ0258), the Guangdong Natural Science and Foundation (No. 2024A1515011039).

\bibliographystyle{IEEEtran}
\bibliography{ref}

\begin{IEEEbiography}[{\includegraphics[width=1in,height=1.25in,clip,keepaspectratio]{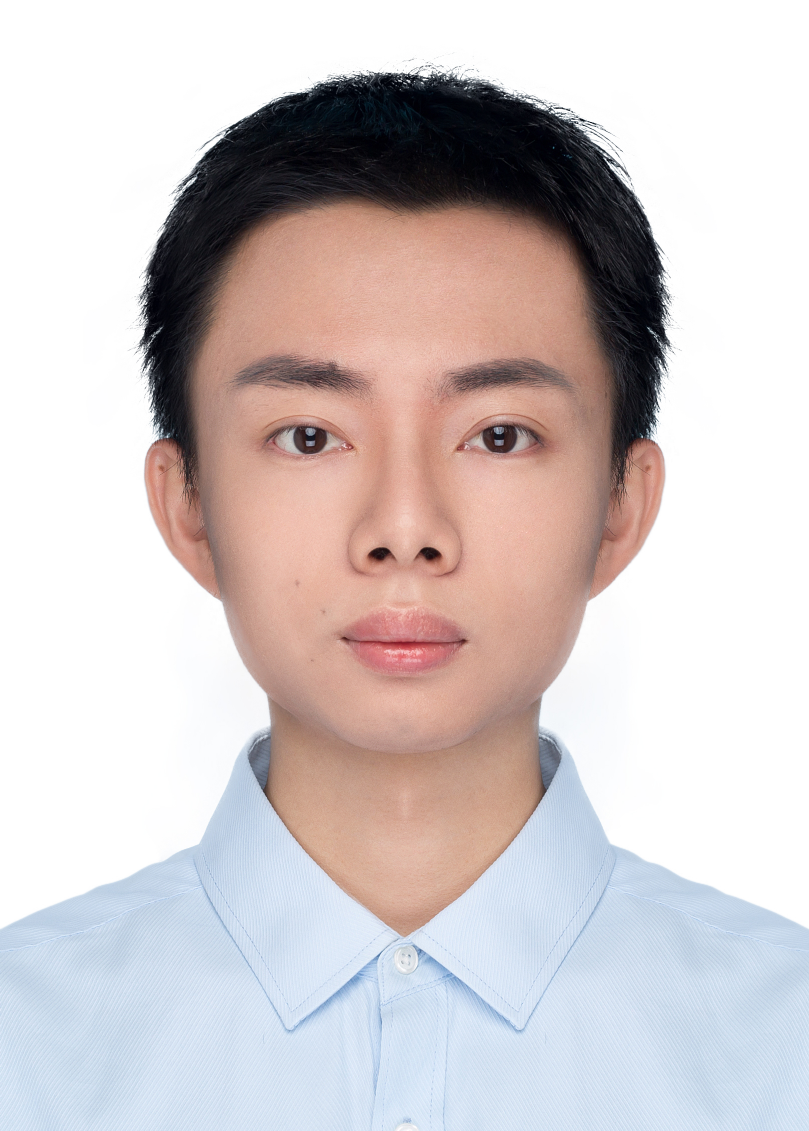}}]{Guotao Xu}
received the B.S. degree in Information Security from Guangdong University of Technology, Guangzhou, China, in 2023. He is currently working toward the M.S. degree in Guangzhou Institute of Technology, Xidian University, Guangzhou, China. His current research interest is privacy-preserving computation.
\end{IEEEbiography}

\begin{IEEEbiography}[{\includegraphics[width=1in,height=1.25in,clip,keepaspectratio]{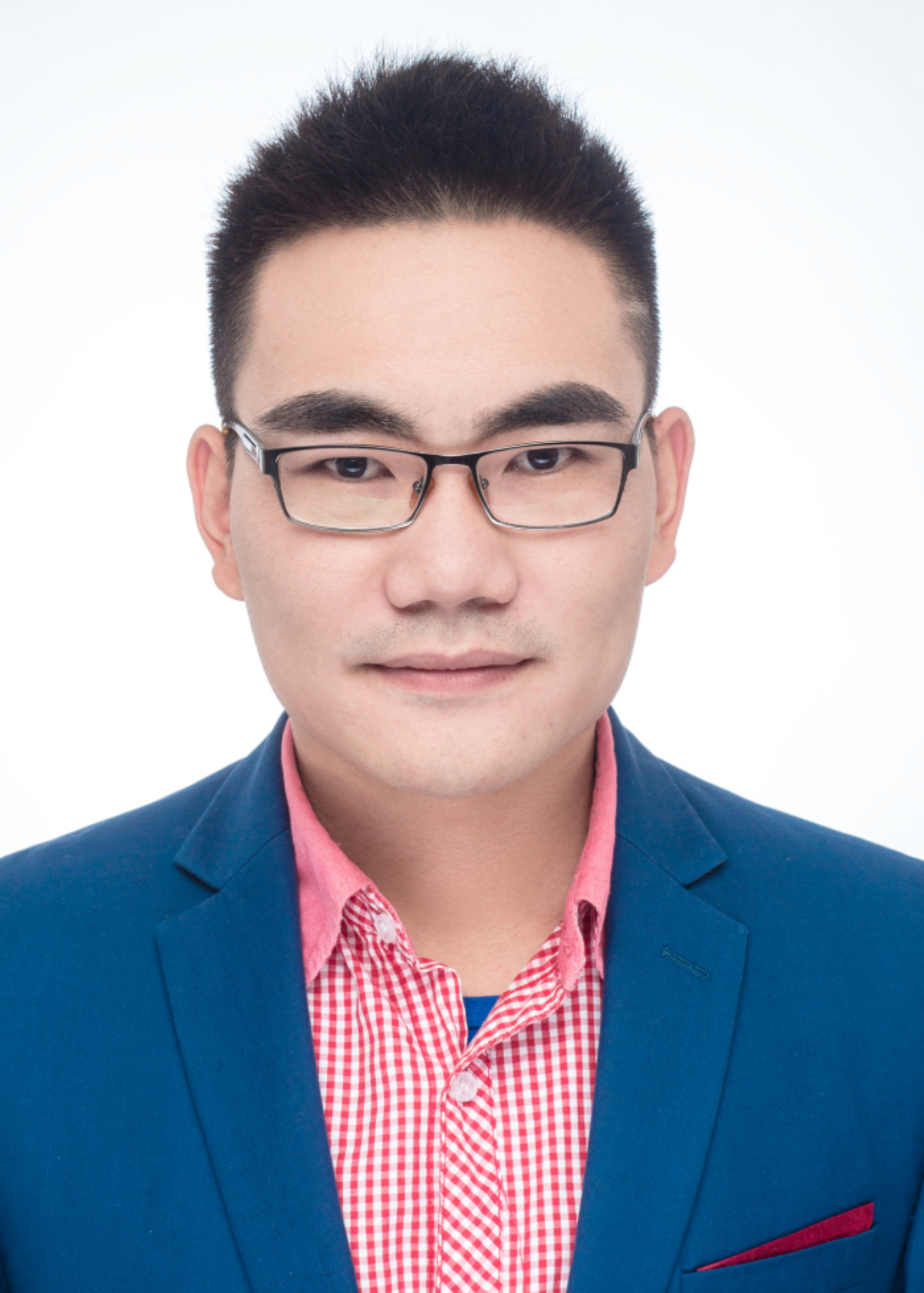}}]{Bowen Zhao}
(Member, IEEE) received the Ph.D. degree in cyberspace security from the South China University of Technology, China, in 2020. He was a Research Scientist with the School of Computing and Information Systems, Singapore Management University, from 2020 to 2021. He is currently an Associate Professor with the Guangzhou Institute of Technology, Xidian University, Guangzhou, China. His current research interests include privacy-preserving computation and learning and privacy-preserving crowdsensing.
\end{IEEEbiography}

\begin{IEEEbiography}[{\includegraphics[width=1in,height=1.25in,clip,keepaspectratio]{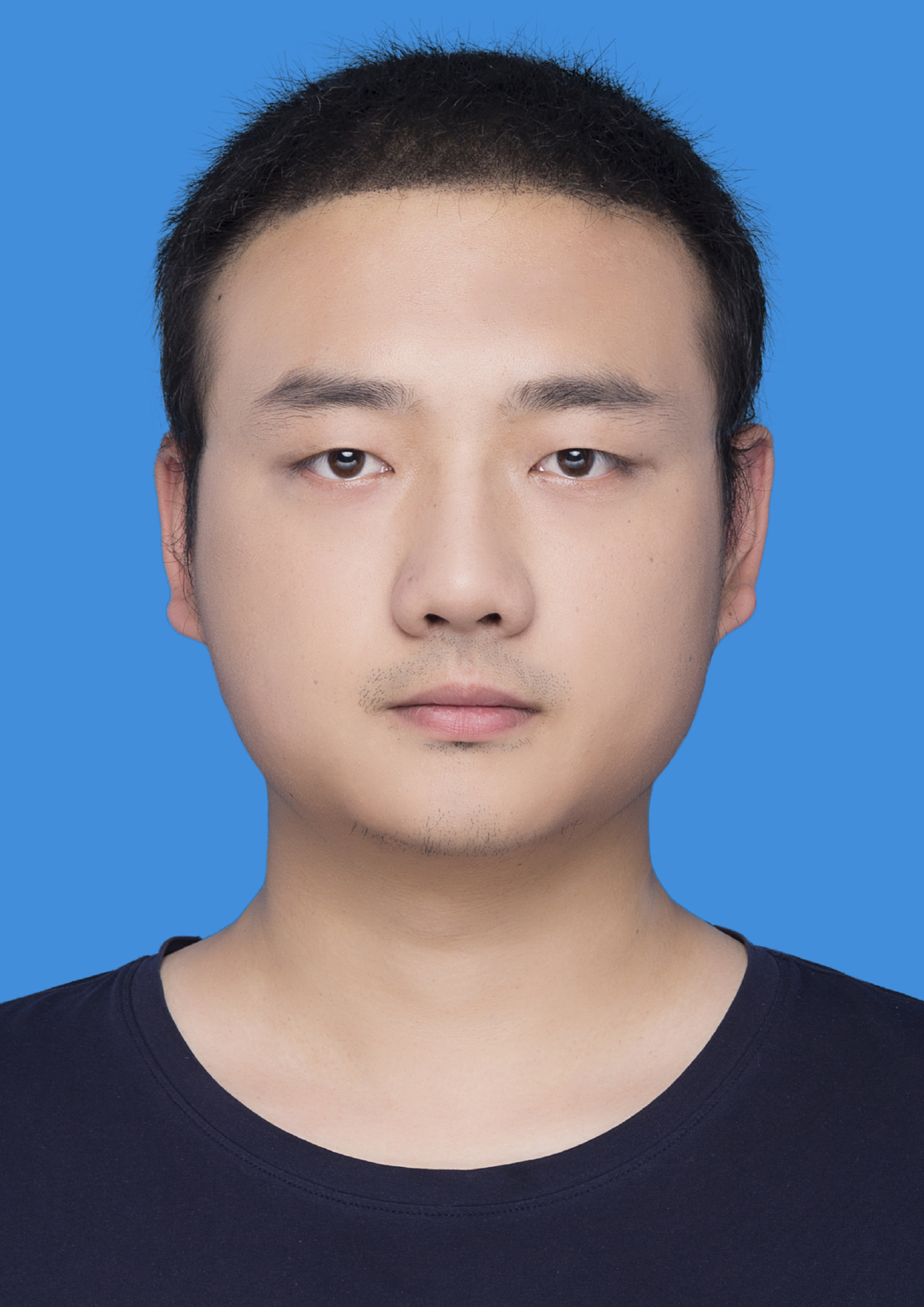}}]{Yang Xiao}
(Member, IEEE) received the B.S. and Ph.D. degrees in communication engineering from Xidian University, Xi’an, China, in 2013 and 2020, respectively. From 2017 to 2019, he was supported by the China Scholarship Council to be a visiting Ph.D. student with the University of New South Wales, Sydney, NSW, Australia. He is currently a Lecturer with the State Key Laboratory of Integrated Services Networks, School of Cyber Engineering, Xidian University. His research interests include social networks, joint recommendations, graph neural networks, trust evaluation, and blockchain.
\end{IEEEbiography}

\begin{IEEEbiography}
    {Yantao Zhong}\!\!\!, China Resources Intelligent Computing Technology (Guangdong) Co., Ltd., Senior Engineer, 1980-, from Shangrao, Jiangxi, focusing on cryptography and privacy computing technology.
\end{IEEEbiography}

\begin{IEEEbiography}
    {Liang Zhai}\!\!\!, Chinese Academy of Surveying and Mapping, Beijing, China.
\end{IEEEbiography}

\begin{IEEEbiography}[{\includegraphics[width=1in,height=1.25in,clip,keepaspectratio]{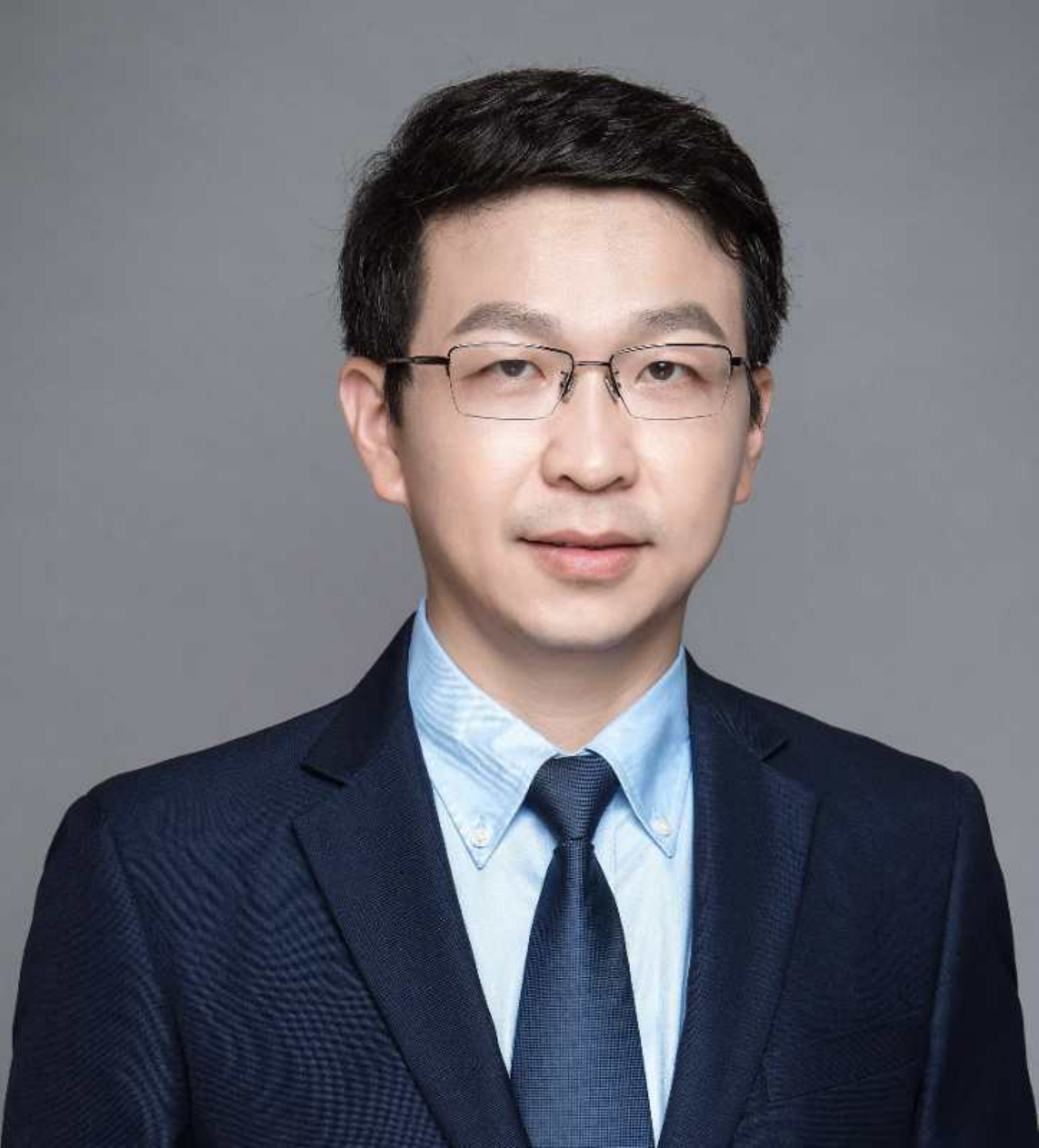}}]{Qingqi Pei}(Senior Member, IEEE) received the B.S., M.S., and Ph.D. degrees in computer science and cryptography from Xidian University, Xi'an, China, in 1998, 2005, and 2008, respectively. He is currently a Professor and a member of the State Key Laboratory of Integrated Services Networks, also a Professional Member of ACM, and a Senior Member of the Chinese Institute of Electronics and China Computer Federation. His research interests focus on digital contents protection and wireless networks and security.
\end{IEEEbiography}

\end{document}